\documentclass[11pt]{article}
\usepackage{hyperref}
\usepackage[english]{babel}
\usepackage[left=1in,right=1in,top=1in,bottom=1.5in]{geometry} 
\usepackage{amssymb,amsmath, amsthm, amscd,ifthen}
\usepackage{graphicx}
\usepackage{xcolor}
\usepackage{bbm}
\usepackage{dsfont}
\usepackage{comment}
\usepackage{framed}
\usepackage{multirow}
\usepackage{transparent}
\usepackage{tikz}
\usepackage{color}
\usepackage{dirtytalk}
\usepackage{mathtools}
\usepackage{framed}
\usepackage{array}
\usepackage{soul}
\usepackage[margin=1cm]{caption}
\RequirePackage[numbers]{natbib}
\usepackage{natbib}
\bibpunct{[}{]}{,}{n}{,}{,}

\makeatletter
\def\BState{\State\hskip-\ALG@thistlm}
\makeatother

\DeclarePairedDelimiter\paren{\lparen}{\rparen}

 \newcommand{\E}{\mathbb{E}}
 \newcommand{\ind}{\mathds{1}}

 \newcommand{\Ind}{\mathbbm{1}} 
 \newcommand{\F}{\mathcal{F}}
  \newcommand{\G}{\mathcal{G}}

  \newcommand{\X}{\mathcal{X}}

\DeclareMathOperator*{\argmin}{arg\,min}

\newcommand{\D}{\mathcal{D}}

\newcommand{\VC}{{\rm VC}}

\definecolor{cd60952}{RGB}{214,9,82}

\date{}

\newtheorem{Theorem}{Theorem}[section]
\newtheorem{Lemma}[Theorem]{Lemma}

\newtheorem{Definition}[Theorem]{Definition}

\newtheorem{Proposition}[Theorem]{Proposition}

\newtheorem{Corollary}[Theorem]{Corollary}
\newtheorem{Remark}[Theorem]{Remark}
\newtheorem{Example}[Theorem]{Example}
\newtheorem{algorithm}[Theorem]{Algorithm}

\title{Fast classification rates without standard margin assumptions}
\author{Olivier Bousquet \and Nikita Zhivotovskiy\thanks{Google Research, Brain Team, Z\"{u}rich \href{}{\{obousquet,zhivotovskiy\}@google.com}}}
\begin{document}
\maketitle
\abstract{
We consider the classical problem of learning rates for classes with  finite VC dimension. It is well known that fast learning rates up to $O\left(\frac{d}{n}\right)$ are achievable by the empirical risk minimization algorithm (ERM) if low noise or margin assumptions are satisfied. These usually require the optimal Bayes classifier to be in the class, and it has been shown that when this is not the case, the fast rates cannot be achieved even in the noise free case. In this paper, we further investigate the question of the fast rates under the misspecification, when the Bayes classifier is not in the class (also called the {\em agnostic} setting).

First, we consider classification with a reject option, namely Chow's reject option model, and show that by slightly lowering the impact of \emph{hard instances}, a learning rate of order $O\left(\frac{d}{n}\log \frac{n}{d}\right)$ is always achievable in the agnostic setting by a specific learning algorithm. Similar results were only known under special versions of margin assumptions. We also show that the performance of the proposed algorithm is never worse than the performance of ERM.

Based on those results, we derive the necessary and sufficient conditions for classification (without a reject option) with fast rates in the agnostic setting achievable by improper learners. This simultaneously extends the work of Massart and N\'{e}d\'{e}lec (Ann. of Statistics, 2006), which studied this question in the case where the Bayesian optimal rule belongs to the class,
and the work of Ben-David and Urner (COLT, 2014), which allows the misspecification but is limited to the no noise setting. Our result also provides the first general setup in statistical learning theory in which an improper learning algorithm may significantly improve the learning rate for non-convex losses.
}

\section{Introduction}
In the context of statistical learning, the notion of the Vapnik-Chervonenkis (VC) dimension \cite{Vapnik68} plays a central role as it characterizes the classes of binary functions for which learning is at all possible, through the finiteness of this dimension.
For classes $\F$ of binary functions with finite VC dimension $d$, one of the key questions is the ability to {\em estimate} the best  function in the class from a finite sample of size $n$. Typically, one looks for high-probability upper bounds on the so-called {\em estimation error} (also called the {\em excess risk}).
More precisely, given a class $\mathcal{F}$ of $\{0,1\}$-valued functions of VC dimension $d$, based on an i.i.d. training sample $\{(X_1, Y_1), \ldots, (X_n, Y_n)\}$, the aim is to construct a classifier $\widehat{f}$ such that its \emph{excess risk} 
\begin{equation}
\label{excessriskfirst}
\mathcal{E}(\widehat{f}) = \Pr(\widehat{f}(X) \neq Y) - \inf\limits_{f \in \mathcal{F}}\Pr(f(X) \neq Y)
\end{equation}
is small with high probability, where the misclassification risk $\Pr(f(X) \neq Y)$ is measured with respect to the same distribution from which the training sample was obtained.

Without any specific assumptions on $\F$ or on the distribution of the data, one can get the following classical result (see \cite{Vapnik74, Talagrand94, Boucheron05b}): for any $\delta>0$, with probability at least $1 - \delta$, with respect to the realization of the learning sample, it holds that
\begin{equation}
\label{standardbound}
\Pr(\widehat{f}(X) \neq Y) - \inf\limits_{f \in \mathcal F}\Pr(f(X) \neq Y) \lesssim \sqrt{\frac{d + \log \frac{1}{\delta}}{n}},
\end{equation}
where the symbol $\lesssim$ suppresses the multiplicative absolute constants and $\widehat{f}$ is an empirical risk minimizer\footnote{In the case where there are several empirical risk minimizers we may choose any of them.} (ERM) defined as $\widehat{f} = \argmin_{f \in \mathcal{F}}\sum_{i = 1}^n\ind[f(X_i) \neq Y_i]$. Here $\ind$ stands for the indicator function. It is well-known that this algorithm is optimal, in the sense that, without any assumptions, ERM achieves the best possible estimation rate $O(\sqrt{d/n})$. In the context of classification the conditions under which this rate can be improved from a slow $O\left(1/\sqrt{n}\right)$ to a so-called fast rate of $O\left(1/n\right)$ (or $O\left(1/n^\alpha\right)$ for $\alpha \in [\frac{1}{2}, 1]$) have been studied thoroughly in the seminal papers \citep{Tsybakov04, Massart06, Bartlett06}.
However, the picture is not quite complete yet and we further investigate the particular conditions under which the fast rates can be obtained.

Note that the estimation error, or the excess risk in Equation \eqref{excessriskfirst} is defined with respect to the best function in the class $\F$. The best possible risk in $\F$ could be significantly larger than the best possible risk achieved by any binary function, which is called the {\em Bayes risk} $R_B=\inf_{f:\X\to\{0,1\}} \Pr(f(X)\ne Y)$ (where the infimum is over all measurable functions) with associated {\em Bayes classifier} $f_B^*$ defined by  $f_B^*(x)=\ind[\Pr(Y=1|x)\ge \frac{1}{2}]$.

While there are well known conditions on the {\em noise function} $\eta$ defined by $\eta(x) = |2\Pr(Y=1|x)-1|$ which allow to obtain fast rates in the case where the model is {\em well specified}, i.e., in the case where the Bayes classifer $f_B^*$ belongs to $\F$ \citep{Tsybakov04, Massart06}, we demonstrate that it is possible to also obtain fast rates in the case where the Bayes classifier is not  in $\F$. This setting is called the {\em misspecified} or the {\em agnostic} setting. One of our main finding is that in this case, the ERM algorithm may not be optimal and we thus exhibit an improper\footnote{A {\em proper} algorithm is one that returns a classifier from $\F$, while an {\em improper} one is allowed to return any binary-valued function.} algorithm that achieves fast estimation rates under appropriate noise conditions.

\subsection{Estimation Rates in Classification}

\paragraph{Estimation and approximation.}
Let us recall that the {\em approximation error} is the deterministic quantity which depends only on the class $\F$ defined by the difference between the best possible risk in $\F$ and the best possible risk over all measurable functions. In other words, it is the non-negative quantity
\begin{equation}\label{approxerror}
    \inf\limits_{f \in \mathcal{F}}R(f) - R(f_B^*)\,,
\end{equation}
where we set $R(f) = \Pr(f(X) \neq Y)$.
Obviously, the fact that the model $\F$ is well-specified is equivalent to the fact that the approximation error is equal to $0$.
Since we don't restrict ourselves to the well-specified setting, to get a complete picture of the behaviour of an algorithm, it would be natural to ask about the value of the approximation error. However, this would require to make assumptions on the class of function and on the distribution which are more specific than the ones we are considering in this study: our main focus is to obtain the fast rates under minimal possible assumptions. It is however common to apply results on the estimation error to the setup of {\em model selection} where one considers a sequence of increasingly large VC classes that eventually can approximate any binary function, but applying our results to this setting is beyond the scope of this paper. We refer to \citep{lugosi2004complexity, Tsybakov04} and to the survey of some related results in the context of classification \cite{Boucheron05b}.

\paragraph{Known results.}
Let us first explain what is known about the estimation rates in classification.
This is (roughly) summarized in Table \ref{tab:main} where we show the rates that can be obtained in the various settings.

\begin{table}[]
    \centering
    \begin{tabular}{|c!{\vrule width 1pt}p{0.24\textwidth}|p{0.20\textwidth}|c|}
    \hline
        \multirow{2}{*}{} & Bayes in $\F$ & \multicolumn{2}{c|}{Agnostic (misspecified)} \\
        \cline{3-4}
        & (well-specified) & Finite Diameter $D$ & Infinite Diameter \\
        \noalign{\hrule height 1pt}
        Deterministic labels& Realizable setting \newline $\tilde{\Theta}\paren*{\frac{d}{n}}$ \cite{Vapnik74, blumer1989learnability}& \multirow{2}{*}{$\tilde{O}\paren*{\frac{D}{n}}$  \cite{Bendavid14}}& \multirow{4}{*}{$\Omega\left(\frac{1}{\sqrt{n}}\right)$\citep{Kaariainen05}}\\
        \cline{1-3}
        Bounded noise & Massart's setting \newline$\tilde{\Theta}\paren*{\frac{d}{hn}}$ \cite{Massart06}& {\bf Our result}\newline $\tilde{O}\paren*{\frac{D}{hn}}$ (Alg.~\ref{alg:finitediameter})&\\
        \hline
        Arbitrary noise & \multicolumn{3}{c|}{ $\Theta\paren*{\sqrt{\frac{d}{n}}}$ \cite{Vapnik74, Talagrand94, devroye2013probabilistic}}\\
        \hline
    \end{tabular}
    \caption{Our main results in the context of the standard setup. Estimation rates for classes with VC dimension $d$ and the combinatorial diameter $D$ and Massart's margin parameter $h > 0$. The optimal rates (up to logarithmic factors) are achieved by ERM in the well-specified and arbitrary noise cases, while in the finite diameter case the bounds are achieved by the algorithm described in Theorem 12 of \cite{Bendavid14} and our Algorithm \ref{alg:finitediameter} respectively. The notions $\tilde{\Theta}(\cdot)$ and $\tilde{O}(\cdot)$ are the versions of the standard ${\Theta}(\cdot)$ and ${O}(\cdot)$ notions that suppress the poly-logarithmic factors.} 
    \label{tab:main}
\end{table}

To simplify the exposition of these results, we only consider the {\em expected} estimation error as opposed to deviations inequalities of the form \eqref{standardbound}. In the most general case, as explained above, without specific assumption, when the class $\F$ is of VC dimension $d$, the expected estimation error is controlled by $\Theta\left(\sqrt{d/n}\right)$ and this rate is obtained by the empirical risk minimizer (ERM).

One can then consider two kinds of restrictions: restrictions on the distribution (and, in particular, on the noise function $\eta$), and restrictions on the class of functions $\F$. 
Regarding the noise function, an extreme situation is one with no noise, that is, when $\eta(X)=1$ almost surely, which is also called the {\em deterministic labeling} setting. Indeed, this means that the label of any particular instance is not random.
A less extreme situation is the {\em bounded noise} regime when the noise is bounded away from $0$, that is, when $\eta(X)\ge h$ almost surely for some $h>0$.

Regarding the class $\F$, the {\em well-specified case} corresponds to the assumption that $f_B^*\in \F$, and the {\em agnostic} or the {\em misspecified case} corresponds to the absence of this assumption. Finally, we can also distinguish the case where the class $\F$ has a finite or infinite {\em combinatorial diameter} $D$ defined as the largest cardinal of a set on which two functions in $\F$ may disagree (see Definition \ref{def:diameter} below).

Let us recall some known results under a combination of those restrictions.
\begin{itemize}
    \item {\bf Realizable setting.} The so-called {\em realizable} case corresponds to a well-specified model and deterministic labels. It is well known that, up to logarithmic factors, in this case the optimal rate is $\tilde{\Theta}(d/n)$ and it is achieved by the ERM algorithm \citep{Vapnik74, blumer1989learnability}.
    \item {\bf Massart's setting.} The Massart's margin condition \cite{Massart06} or Massart's setting correspond to the conjunction of the bounded noise and that $f^*_B \in \F$. In this case, the rate is still fast, but depends on the margin parameter $h$. 
More precisely, one can show that the expected estimation error of ERM is upper bounded by $O\left(\frac{d\log\frac{nh^2}{d}}{nh}\right)$. Moreover, matching lower bounds can be shown \cite{Massart06, Zhivotovskiy18}. 
    \item {\bf Deterministic labels.} Recently, Ben-David and Urner \cite{Bendavid14} have studied the deterministic labels case, that is, $Y = f^*_B(X)$ almost surely and $h = 1$, under the model misspecification and have shown that one can obtain the fast rates under the assumption of the finite combinatorial diameter $D$. However, they show that ERM is not optimal as it may suffer the slow rates.
 In particular, there is a class with VC dimension $d$ such that any proper learning algorithm has a learning rate lower bounded by $\Omega\left(\sqrt{\frac{d}{n}}\right)$. Finally, \cite{Kaariainen05, Bendavid14} also show that as soon as the combinatorial diameter $D$ is infinite and the model is misspecified, \emph{any algorithm} may have the slow rate $\Omega\left(\frac{1}{\sqrt{n}}\right)$ even in the case of deterministic labels.
\end{itemize}

The remaining open question is thus about the rates in the bounded noise case under the model misspecification. One of our main results shows the existence of an algorithm (Algorithm \ref{alg:finitediameter}) which obtains fast rates in this situation, thus completing the picture.

\subsection{Our Approach}
In order to study the non-deterministic and agnostic regime, we combine the following insights which ultimately lead to our main result:
\begin{enumerate}
    \item {\bf The problem of multiple (almost) minimizers.} 
    The first fact, which has been observed in the previous literature, is that the main reason why ERM suffers the slow rates in the agnostic regime is due to the fact that there can be multiple almost ERMs that are far apart from each other in terms of their $L_1(P)$ distance. Therefore, in general, there can be many classifiers that seem identically good given the data (hence cannot be told apart) but have different risks and just picking one of the ERM classifiers will not necessarily work. So one could imagine to combine several of them, but this means that we have to come up with a way to decide on how to predict on instances where they disagree.
    \item {\bf Aggregation for strongly convex losses.} The second insight is that if one considers classes $\F$ of real-valued functions and an $\ell_q, q > 1$ risk, the combination of the Lipschitzness and strong convexity of the loss function allows to design algorithms that achieve fast estimation rates with respect to that loss function without any additional assumptions rather than the boundedness of $Y$ and the boundedness of functions in $\F$. This is at the heart of the aggregation theory. 
    \item {\bf Classification with a reject option (with abstentions).} The third insight is coming from the idea of classification with a reject option, or abstaining classifiers which are allowed to abstain from making a prediction on certain instances.  One can see such classifiers as functions taking their values in $\{0,1,*\}$ and can define the following loss function:
\begin{equation}
\label{rp}
R^p(f) = \Pr\left(f(X) \neq Y\ \text{and}\ f(X) \in \{0, 1\}\right) + \left(\frac{1}{2} - p\right)\Pr(f(X) = *),
\end{equation}    
    where $p \in [0, \frac{1}{2}]$ is a fixed parameter corresponding to the cost of not making a definite prediction. The key observation that we use is that if we replace the value $*$ by $1/2$ and consider abstaining classifiers as functions to the real line, this loss function is both Lipschitz and strongly convex and can be related to $\ell_q$ risk for a specific value of $q$. This allows us to apply the techniques of the aggregation theory to analyze it. Therefore, based on the aggregation result the values of $p$ marginally greater than zero allow us to obtain the fast rates without any additional assumptions.
    \item {\bf Breaking ties.} We will put the above ingredients together by coming up with an aggregation strategy over a set of almost ERMs. Finally, we show how to convert an abstaining algorithm into a non-abstaining one. For that, we need to be able to perform confident predictions on the region of disagreement of almost ERMs by the majority voting.
\end{enumerate}

Let us now review the aggregation theory and reject model in more detail.

\paragraph{Aggregation theory.}
The aggregation theory was initiated by Nemirovskii \cite{Nemirovski00}. Assume that we are given a finite dictionary $F$ of real-valued functions and denote $M = |F|$. The general problem in \emph{model selection aggregation} \cite{Tsybakov03} is to construct an estimator $\widehat{f}$ such that
\begin{equation}
\label{aggreg}
\E(\widehat{f}(X) - Y)^2 - \inf\limits_{f \in F}\E(f - Y)^2 \lesssim \frac{\log M}{n},
\end{equation}
based on the training sample provided that $|Y| \le 1$ and $|f(X)| \le 1$ for all $f \in F$ and the expectation is also taken with respect to the training sample.
Interestingly, since $F$ is finite, and thus not convex, no \emph{proper estimator} (i.e., such that $\widehat{f} \in F$), may achieve the desired bound \eqref{aggreg} \cite{Juditsky08}.  At the same time, there are several ways to construct a class $G$ such that $F \subset G$ and an estimator $\widehat{f} \in G$ such that the learning rate \eqref{aggreg} is satisfied. In the last decade, several optimal algorithms were introduced, see \cite{Audibert07, Lecue09, Lecue14, Rakhlin15, Mendelson17} and references therein. Unfortunately, the results of the aggregation theory are specific to the squared loss (can be extended to strongly convex Lipschitz losses \cite{Audibert09, Lecue09}), but in the binary loss case none of them are applicable: in the theory of classification the lower bounds $\Omega\left(\frac{1}{\sqrt{n}}\right)$ can be provided for \emph{any learning algorithm}. 

 One of the challenges we are facing in adapting results from the aggregation literature is that the algorithms are usually designed for the squared loss only and for finite classes of functions while we are considering classes with finite VC dimension. We also have that the progressive mixture rules \cite{Audibert07}, as well as the variants of the \emph{empirical star-algorithm} considered  in \cite{Audibert07, Lecue09, Rakhlin15}, may produce arbitrary convex combinations of a few functions in the class $F$ \eqref{aggreg} which is not directly applicable to our situation of binary classification.
So, in order to design an appropriate aggregation procedure, we exploit and further develop the following recent observation of Mendelson \cite{Mendelson17, Mendelson18}: in the model selection aggregation with the squared loss the class $G$ introduced above may be restricted only to the functions of the form $\frac{f + g}{2}$ for some  $f, g \in F$, as long as an appropriate learning algorithm is used.

\paragraph{Classification with a reject option.}
This setup has been motivated by various practical constraints, and many authors have studied its statistical properties (see e.g., \cite{Wegkamp06, Bartlett08, Ming10, Elyaniv10} and references therein). One of the first reject option models appears in the seminal paper of Chow \cite{Chow70} (see the discussion in \cite[Section 9]{Elyaniv10} and the analysis of this framework in \cite{Wegkamp06, Bartlett08}).
The intuitive interpretation of the risk functional \eqref{rp} is the following: we pay the fixed cost $\frac{1}{2} - p$ for all instances where we decided to abstain from the prediction. This cost is only slightly lower than the expected risk of the random guess if $p$ is close to zero. In particular, it can be seen (see e.g., the discussion in \cite{Wegkamp06}) that the case $p = 0$ corresponds the standard agnostic learning. Indeed, it follows from the result of \citep{Chow70} that the $\{0, 1, *\}$-valued minimizer of the $R_p$ risk is equal to $f^{*}_B$ except on the instances where $1/2-p \le \Pr(Y = 1|x) \le  1/2 + p$ which corresponds to the instance where it abstains. Another natural way to interpret this risk model is as follows: if on the test sample, the cost of the misclassification is too high, one should ask an expert to classify the instance. The reason why we are not is not using the expert all the time, and the actual learning is happening is because the expert may predict only marginally better than a random guess. Therefore, if we want to construct a classifier with a reject option $\tilde{f}$ such that $R^p(\tilde{f}) \ll \frac{1}{2} - p$, then only a small fraction of the instances can be assigned with $*$. Interestingly, a value of $p$ only marginally greater than zero is sufficient to eliminate the slow rates. In Section \ref{sec:relwork} we make some additional comparisons with reject option models.

\paragraph{Subsequent work.} Inspired by the results of the present work, a similar phenomenon have been recently observed in the online learning setup in \cite{neu2020}. Their results demonstrate that abstentions with the cost only marginally lower than the random guess lead to the fast rates without any additional assumptions. However, the proof technique and the algorithm in \cite{neu2020} are different. Instead of using a modified ERM algorithm, the authors use the modified exponential weighting algorithm.

\subsection{Structure of the Paper}
\begin{enumerate}
\item In Section \ref{sec:fastraetsabst}, we formulate our Theorem \ref{mainthmproc} on classification with a reject option. We show that by allowing the learner to predict marginally better than a random guess on some hard instances, fast classification rates are possible without any additional assumptions.
\item In Section \ref{sec:ermrates} we show that our algorithm with a reject option always recovers the rates of ERM if the region of abstentions is replaced by the random guesses.
\item In Section \ref{sec:fasttartesmisspecificiation}, we apply Theorem \ref{mainthmproc} to the analysis of VC classes with the finite combinatorial diameter. Our result is informally shown in Table \ref{tab:main}. In particular, we develop the first general statistical setup showing that an improper learning algorithm may significantly improve the learning rates for the binary non-convex loss. In Section \ref{sec:distribdepdiameter} we discuss a distribution-dependent version of the combinatorial diameter and a corresponding risk bound.
\item In Section \ref{app:log}, we prove a simple result that improves the logarithmic factor in the risk bound for VC classes with the finite combinatorial diameter in the case of deterministic labeling.
\item Section \ref{sec:proofofabstthm} and Section \ref{sec:remainingproofs} are devoted to the proofs. Among the new observations is the fact that in the binary classification, the classical bounds under Massart's and Tsybakov's noise assumptions \cite{Massart06, Tsybakov04} can be recovered using only the original bounds of Vapnik and Chervonenkis \cite{Vapnik74}.
\item Section \ref{sec:relwork} is devoted to the remaining comparisons with some previous papers.
\end{enumerate}

\subsection{Basic Notation and Setup}
We introduce some notation and basic definitions that will be used throughout the text. The symbol $\Ind[A]$ denotes an indicator function of the event $A$. The notation $f \lesssim g$ or $g \gtrsim f$ means that for some universal constant $c>0$ we have $f \le cg$. To avoid the problems with the logarithmic function we assume that $\log x$ means $\max\{\log x, 1\}$. Throughout the paper we also use the standard $O(\cdot), \Omega(\cdot), \Theta(\cdot)$ notation.
We use the letters $c, c_1, c_2. \ldots$ to denote the numerical constants that can change from line to line. For a real valued function $g$ and $s \ge 1$ we define the $L_s$ norm as $\|g\|_{L_s} = \left(\E |g(Z)|^s\right)^\frac{1}{s}$. 

We define the instance space $\mathcal{X}$ and the label space $\mathcal{Y} = \{0, 1\}$. We assume that the set $\mathcal{X} \times \mathcal{Y}$ is equipped with some $\sigma$-algebra and a probability measure $P = P_{X, Y}$ on measurable subsets is defined. We also assume that we are given a set of classifiers $\mathcal F$: these are measurable functions with respect to the introduced $\sigma$-algebra, mapping $\mathcal{X}$ to $\mathcal{Y}$. A learner observes $\left\{(X_{1}, Y_{1}), \ldots, (X_{n}, Y_{n})\right\}$, an i.i.d. training sample from an unknown distribution $P$. Also denote $Z_i = (X_i,Y_i)$ and $\mathcal{Z} = \mathcal{X} \times \mathcal{Y}$. By $P_{n}$ we will denote the expectation with respect to the empirical measure (empirical mean) induced by this sample. We say a set $\{x_1, \ldots, x_k\} \in \mathcal{X}^{k}$ is shattered by $\mathcal F$ if there are $2^k$ distinct classifications of $\{x_1, \ldots, x_k\}$ realized by classifiers in $\mathcal F$. The VC dimension of $\mathcal F$ is the largest integer $d$ such that there exists a set $\{x_1, \ldots, x_d\}$ shattered by $\mathcal F$ \cite{Vapnik68}. We define the growth function $\mathcal{S}_{\mathcal{F}}(n)$ as the maximum possible number of different classifications of a set of $n$ instances realized by classifiers in $\mathcal F$ (maximized over the choice of the $n$ instances). A famous bound by Vapnik and Chervonenkis \cite{Vapnik68} implies that $\log \mathcal{S}_{\mathcal{F}}(n) \lesssim d\log \frac{n}{d}$. We use this basic result throughout the text without referring to it.
 Analogously, we define the notion of the growth function for discrete-valued classes of functions: in this case it is equal to the maximum possible number of different projections of the set of function on a set of $n$ instances. As above, define the prediction risk as $R(f) = \Pr(f(X) \neq Y)$. The Bayesian optimal rule $f^*_{B}$ is defined by $f^*_{B}(x)=\ind[\Pr(Y=1|x)\ge \frac{1}{2}]$. The noise function $\eta$ is defined by $\eta(x) = |2\Pr(Y=1|x)-1|$. The smallest $h \ge 0$ such that 
\begin{equation}
\label{eq:massartsmargin}
|2\Pr(Y=1|X)-1| \ge h
\end{equation}
almost surely is called the Massart's margin parameter \citep{Massart06}. Slightly abusing the notation the symbol $h$ will sometimes denote a real valued function in our proofs. Throughout the paper we set
\[
f^* = \argmin_{f \in \F}R(f),
\]
and assume the uniqueness of $f^*$ without loss of generality. The empirical risk is defined as 
\[
R_n(f) = \frac{1}{n}\sum\limits_{i=1}^n\ind[f(X_i) \neq Y_i].
\]
The $R^p$ risk is given by \eqref{rp}. We adopt the standard assumption that the events appearing in probability claims are measurable.

\section{Fast Rates for Classification with a Reject Option}
\label{sec:fastraetsabst}
We start this section with our first main result. The full proof is deferred to Section \ref{sec:proofofabstthm}.

\begin{Theorem}
\label{mainthmproc}
Fix $p \in [0, \frac{1}{2})$ and consider a class $\mathcal{F}$ of binary functions with VC dimension $d$. For any $\delta\in (0,1)$, there is a learning algorithm (namely, Algorithm \ref{alg:abstentionalgorithm} below) with its output classifier denoted by $\widehat{f_p}$, taking its values in $\{0, 1, *\}$ such that, with probability at least $1 - \delta$,
\[
R^p(\widehat{f}_p) - R(f^*) \lesssim \frac{d\log\frac{n}{d} + \log \frac{1}{\delta}}{np}.
\]
\end{Theorem}

Before we proceed with the formal definition of the algorithm, we discuss some related results. First, compare the $O(\frac{1}{n})$ convergence of Theorem \ref{mainthmproc} with the $O\left(\sqrt{\frac{1}{n}}\right)$ bound \eqref{standardbound} that holds for ERM in the same fully agnostic model. Second, we notice that the parameter $p$ in our bound plays a role similar to the role of the margin parameter $h$ in the seminal work \cite{Massart06}. Indeed, in that paper the authors assume the margin assumption \eqref{eq:massartsmargin}, that is $\eta(X) = |2\E[Y|X] - 1| \ge h$ almost surely for $h \ge 0$, and that the Bayes optimal classifier $f^*_B$ belongs to the class $\mathcal{F}$. Under these assumptions, provided that $h > 0$ any ERM $\widehat{f}$ satisfies, with probability at least $1 - \delta$,
\begin{equation}
\label{massartriskbound}
R(\widehat{f}) - R(f^*_B) \lesssim \frac{d\log\frac{nh^2}{d} + \log \frac{1}{\delta}}{nh}.
\end{equation}

It appears that the result \eqref{massartriskbound} is limited to the restrictive assumption that $f^*_B$ belongs to $\mathcal F$. K{\"a}{\"a}ri{\"a}inen \cite{Kaariainen05} and later Ben-David and Urner \cite{Bendavid14} showed that the universal (valid for any learning algorithm) lower bound of order $\Omega\left(\frac{1}{\sqrt{n}}\right)$ holds for $R(\tilde{f}) - R(f^*)$ in some cases even if $Y = f^*_B(X)$ which corresponds to the most favourable case $h = 1$. The result of our Theorem \ref{mainthmproc} avoids this obstacle but replaces the margin condition $\eta(X) > h$ by the option to abstain from some predictions and uses an estimator that is different from ERM.

Further, observe that both Theorem \ref{mainthmproc} and \eqref{massartriskbound} only require that $p$ and $h$ are marginally larger than zero in order to get $O\left(\frac{1}{n}\right)$ upper bound. As an illustration of this, one may fix, say, $p = 0.01$ in Theorem \ref{mainthmproc} which corresponds to the abstention cost $0.49$. This cost is only marginally smaller than the average cost of the random guess.

\subsection{An Algorithm that Abstains on Hard Instances}
For the rest of this section we introduce and discuss a classifier with a reject option, which appears to have the following properties:
\begin{enumerate}
\item Our abstention regions have a simple interpretation: the learner refuses to predict and outputs $*$ if, based on the available information, there is nothing better than a random guess when trying to predict as well as $f^*$. 
\item By replacing $*$ by a prediction with risk marginally better than $\frac{1}{2}$ on average, which corresponds to Chow's reject option model with $0 < p <  \frac{1}{2}$ \eqref{rp}, we prove that our algorithm provides fast learning rates for the excess risk \eqref{excessriskfirst} in a completely agnostic scenario. This is the result of Theorem \ref{mainthmproc}.
\item By replacing $*$ by a random guess, we recover the well-known learning rates in the agnostic case as well as faster rates if the appropriate low noise assumption holds. 
\end{enumerate}
We need the following definition.
\begin{Definition}[Almost Empirical Risk Minimizers]\label{def:aerm}
Given a class $\F$ of the $\VC$ dimension $d$, a sample of size $n$ and a confidence level $\delta>0$, we set $\alpha = \alpha(n, d, \delta) = \sqrt{\frac{d\log \frac{n}{d} + \log(1/\delta)}{n}}$ and define the set of {\em almost empirical risk minimizers} as
    \begin{equation}
    \label{hatF}
    \widehat{\F} = \widehat{\F}(c,n, d, \delta,\widehat{g}) = \left\{f \in \F: R_n(f) - R_n(\widehat{g}) \le c(\alpha^2 + \alpha \sqrt{P_n|\widehat{g} - f|})\right\},
    \end{equation}
where $\widehat{g}$ is some empirical risk minimizer in $\F$.
\end{Definition}

We note that similar sets of almost minimizers appear in the analysis of active learning \citep{Dasgupta07} as well as in the model selection \citep{lugosi2004complexity} and the aggregation \citep{Lecue09, Mendelson17} literature. Finally, for a $\{0, 1, *\}$-valued function $f$ we introduce the empirical version of the $R^p$ risk:
\begin{equation}
\label{eq:emprprisk}    
R_n^p(f) = \frac{1}{n}\sum\limits_{i = 1}^{n}\left(\ind[f(X_i) \neq Y_i\ \text{and}\ f(X_i) \in \{0, 1\}] + \left(\frac{1}{2} - p\right)\ind[f(X_i) = *] \right).
\end{equation}
Finally, we are ready to present our algorithm. 
\begin{framed}
\begin{algorithm}
\label{alg:abstentionalgorithm}
Let $c > 0$ be a specifically chosen numerical constant in Definition \ref{def:aerm}. Given a class $\F$ of binary functions of VC dimension $d$, a sample of size $2n$, a confidence level $\delta$, and an abstention level $p\in[0,\frac{1}{2}]$, proceed as follows:
\begin{itemize}
    \item Based on the first half of the sample $\{(X_i, Y_i)_{i = 1}^n\}$ denote ERM with respect to the binary loss by $\widehat{g}$.
    \item Based on $\{(X_i, Y_i)_{i = 1}^{n}\}$ construct the class of almost empirical minimizers $\widehat{\F}$ as in Definition \ref{def:aerm}.
    \item Consider the class $\left\{\frac{f + \widehat{g}}{2}: f \in \widehat{\mathcal F}\right\}$ of $\{0,1, \frac{1}{2}\}$-valued functions and convert it into $\{0, 1, *\}$-valued class $\widehat{\mathcal{G}}$ of functions by replacing the values $\frac{1}{2}$ by $*$.
    \item If $p \in [0, \frac{1}{4}]$ output $\widehat{f}_p$ defined as the minimizer of the empirical $R_n^p$ risk over the class $\widehat{\mathcal{G}}$ based on the second half of the sample $\{(X_i, Y_i)_{i = n + 1}^{2n}\}$. 
    \item Otherwise, if $p \in (\frac{1}{4}, \frac{1}{2}]$, set $\widehat{f}_p = \widehat{f}_{1/4}$.
\end{itemize}
\end{algorithm}
\end{framed} 

We remark that our learning algorithm depends on the values $\delta, n, d, p$ as they are required to construct the set $\widehat{\F}$. Although the dependence on $p$ is quite natural, the dependence on the remaining parameters can be potentially removed. Indeed, our analysis is closely related to the analysis of the empirical-star algorithm of Audibert \cite{Audibert07}, which is completely parameter-free.

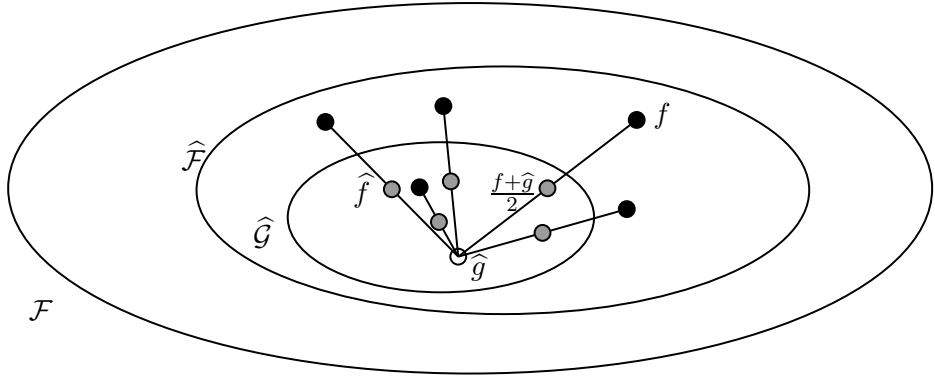
\begin{figure}[h]
    \centering
\tikzset{every picture/.style={line width=0.75pt}} 

\begin{tikzpicture}[x=0.75pt,y=0.75pt,yscale=-1,xscale=1]

\draw   (100,100.5) .. controls (100,48.86) and (204.32,7) .. (333,7) .. controls (461.68,7) and (566,48.86) .. (566,100.5) .. controls (566,152.14) and (461.68,194) .. (333,194) .. controls (204.32,194) and (100,152.14) .. (100,100.5) -- cycle ;
\draw   (195,101.5) .. controls (195,66.98) and (264.17,39) .. (349.5,39) .. controls (434.83,39) and (504,66.98) .. (504,101.5) .. controls (504,136.02) and (434.83,164) .. (349.5,164) .. controls (264.17,164) and (195,136.02) .. (195,101.5) -- cycle ;
\draw   (241,115.13) .. controls (241,94.21) and (275.59,77.25) .. (318.25,77.25) .. controls (360.91,77.25) and (395.5,94.21) .. (395.5,115.13) .. controls (395.5,136.04) and (360.91,153) .. (318.25,153) .. controls (275.59,153) and (241,136.04) .. (241,115.13) -- cycle ;
\draw  [fill={rgb, 255:red, 255; green, 255; blue, 255 }  ,fill opacity=1 ] (323,135) .. controls (323,132.79) and (324.79,131) .. (327,131) .. controls (329.21,131) and (331,132.79) .. (331,135) .. controls (331,137.21) and (329.21,139) .. (327,139) .. controls (324.79,139) and (323,137.21) .. (323,135) -- cycle ;
\draw  [fill={rgb, 255:red, 0; green, 0; blue, 0 }  ,fill opacity=1 ] (256,67) .. controls (256,64.79) and (257.79,63) .. (260,63) .. controls (262.21,63) and (264,64.79) .. (264,67) .. controls (264,69.21) and (262.21,71) .. (260,71) .. controls (257.79,71) and (256,69.21) .. (256,67) -- cycle ;
\draw  [fill={rgb, 255:red, 0; green, 0; blue, 0 }  ,fill opacity=1 ] (315.5,59) .. controls (315.5,56.79) and (317.29,55) .. (319.5,55) .. controls (321.71,55) and (323.5,56.79) .. (323.5,59) .. controls (323.5,61.21) and (321.71,63) .. (319.5,63) .. controls (317.29,63) and (315.5,61.21) .. (315.5,59) -- cycle ;
\draw  [fill={rgb, 255:red, 0; green, 0; blue, 0 }  ,fill opacity=1 ] (413,66) .. controls (413,63.79) and (414.79,62) .. (417,62) .. controls (419.21,62) and (421,63.79) .. (421,66) .. controls (421,68.21) and (419.21,70) .. (417,70) .. controls (414.79,70) and (413,68.21) .. (413,66) -- cycle ;
\draw  [fill={rgb, 255:red, 0; green, 0; blue, 0 }  ,fill opacity=1 ] (408,111) .. controls (408,108.79) and (409.79,107) .. (412,107) .. controls (414.21,107) and (416,108.79) .. (416,111) .. controls (416,113.21) and (414.21,115) .. (412,115) .. controls (409.79,115) and (408,113.21) .. (408,111) -- cycle ;
\draw  [fill={rgb, 255:red, 0; green, 0; blue, 0 }  ,fill opacity=1 ] (303.5,100) .. controls (303.5,97.79) and (305.29,96) .. (307.5,96) .. controls (309.71,96) and (311.5,97.79) .. (311.5,100) .. controls (311.5,102.21) and (309.71,104) .. (307.5,104) .. controls (305.29,104) and (303.5,102.21) .. (303.5,100) -- cycle ;
\draw    (260,67) -- (327,135) ;
\draw    (319.5,59) -- (327,135) ;
\draw    (417,66) -- (327,135) ;
\draw    (307.5,100) -- (327,135) ;
\draw    (412,111) -- (327,135) ;
\draw  [fill={rgb, 255:red, 155; green, 155; blue, 155 }  ,fill opacity=1 ] (368,100.5) .. controls (368,98.29) and (369.79,96.5) .. (372,96.5) .. controls (374.21,96.5) and (376,98.29) .. (376,100.5) .. controls (376,102.71) and (374.21,104.5) .. (372,104.5) .. controls (369.79,104.5) and (368,102.71) .. (368,100.5) -- cycle ;
\draw  [fill={rgb, 255:red, 155; green, 155; blue, 155 }  ,fill opacity=1 ] (319.25,97) .. controls (319.25,94.79) and (321.04,93) .. (323.25,93) .. controls (325.46,93) and (327.25,94.79) .. (327.25,97) .. controls (327.25,99.21) and (325.46,101) .. (323.25,101) .. controls (321.04,101) and (319.25,99.21) .. (319.25,97) -- cycle ;
\draw  [fill={rgb, 255:red, 155; green, 155; blue, 155 }  ,fill opacity=1 ] (313.25,117.5) .. controls (313.25,115.29) and (315.04,113.5) .. (317.25,113.5) .. controls (319.46,113.5) and (321.25,115.29) .. (321.25,117.5) .. controls (321.25,119.71) and (319.46,121.5) .. (317.25,121.5) .. controls (315.04,121.5) and (313.25,119.71) .. (313.25,117.5) -- cycle ;
\draw  [fill={rgb, 255:red, 155; green, 155; blue, 155 }  ,fill opacity=1 ] (289.5,101) .. controls (289.5,98.79) and (291.29,97) .. (293.5,97) .. controls (295.71,97) and (297.5,98.79) .. (297.5,101) .. controls (297.5,103.21) and (295.71,105) .. (293.5,105) .. controls (291.29,105) and (289.5,103.21) .. (289.5,101) -- cycle ;
\draw  [fill={rgb, 255:red, 155; green, 155; blue, 155 }  ,fill opacity=1 ] (365.5,123) .. controls (365.5,120.79) and (367.29,119) .. (369.5,119) .. controls (371.71,119) and (373.5,120.79) .. (373.5,123) .. controls (373.5,125.21) and (371.71,127) .. (369.5,127) .. controls (367.29,127) and (365.5,125.21) .. (365.5,123) -- cycle ;

\draw (109,155.4) node [anchor=north west][inner sep=0.75pt]    {$\mathcal{F}$};
\draw (186,75.4) node [anchor=north west][inner sep=0.75pt]    {$\widehat{\mathcal{F}}$};
\draw (222,113.4) node [anchor=north west][inner sep=0.75pt]    {$\widehat{\mathcal{G}}$};
\draw (332,132.4) node [anchor=north west][inner sep=0.75pt]    {$\widehat{g}$};
\draw (424,55.9) node [anchor=north west][inner sep=0.75pt]    {$f$};
\draw (340.5,91.4) node [anchor=north west][inner sep=0.75pt]    {$\frac{f+\widehat{g}}{2}$};
\draw (272.5,91.4) node [anchor=north west][inner sep=0.75pt]    {$\widehat{f}$};

\end{tikzpicture}
    \caption{Illustration of Algorithm \ref{alg:abstentionalgorithm}: after identifying some ERM $\widehat{g}$ on the first half of the sample and computing the class of almost empirical risk minimizers, $\widehat{\F}$, we compute the mid-points between $\widehat{g}$ and elements of $\widehat{\F}$ to form the set $\widehat{\mathcal{G}}$ and pick the best one $\widehat{f}$ according to the empirical $R^p$ risk on the second half of the sample.}
    \label{fig:agg}
\end{figure}

\subsection{Recovering the Rates of ERM}
\label{sec:ermrates}
In this section we show that the $R^0$ risk (that is, $p = 0$ in \eqref{rp}) of our Algorithm \ref{alg:abstentionalgorithm} satisfies the standard upper bounds for ERM. In this case the loss of all abstentions is essentially replaced by the expected loss of a random guess which is equal to $\frac{1}{2}$. We also show that the improved rates inherent to ERM are also recovered under the margin assumptions. Moreover, margin assumptions imply that the region of abstentions shrinks as the sample size $n$ grows. This shows that apart from the computational issues, our algorithm is preferable to ERM as it performs as good as ERM but shows additionally which instances are responsible for the slow rates.

Observe that $R^p$ is monotonic in $p$. This means, in particular, that for any $\{0, 1, *\}$-valued function $f$ it holds that
\[
R^p(f) - R(f^*) \le R^0(f) - R(f^*).
\]
Therefore, the upper bounds of this section apply automatically to the bounds on $R^p(f) - R(f^*)$.
We first start with the general agnostic case without any assumptions on the noise and show that we get the standard $O\left(\sqrt{d/n}\right)$ rate up to a logarithmic factor.

\begin{Proposition}
\label{prop:mainthmprocermfirst}
Under the conditions of Theorem \ref{mainthmproc}, for any fixed $p \in [0, \frac{1}{2})$, with probability at least $1 - \delta$, it holds that
\begin{equation}
\label{general}
R^0(\widehat{f}_p) - R(f^*) \lesssim \sqrt{\frac{d\log\frac{n}{d} + \log \frac{1}{\delta}}{n}}.
\end{equation}
\end{Proposition}

To formulate the second result of this section we introduce the following standard definition.

\begin{Definition}[The Bernstein Assumption for the binary loss \cite{Bartlett06, Boucheron05b}]
\label{bernstein}
The class of binary functions $\mathcal F$ together with the distribution $P_{X, Y}$ satisfy the $(\beta, B)$-Bernstein assumption if for all $f \in \mathcal F$,
\[
\Pr(f(X) \neq f^*(X)) \le B\left(R(f) - R(f^*)\right)^{\beta},
\]
where $\beta \in [0, 1]$ and $B \ge 1$.
\end{Definition}
It is well known that Tsybakov's and Massart's noise assumptions imply the Bernstein assumption for some particular values of $\beta$ and $B$ (see e.g., \cite[Section 5.2]{Boucheron05b}). In particular, if $f^*_B \in \F$, Massart's assumption $\eta(X) > h$ implies the Bernstein assumption with $B = \frac{1}{h}$ and $\beta = 1$. We are ready to state our general result showing that Algorithm \ref{alg:abstentionalgorithm} is also adaptive to the Bernstein assumption even if all abstentions are replaced by the random guesses.

\begin{Proposition}
\label{prop:mainthmprocermsecond}
Under the conditions of Theorem \ref{mainthmproc}, for any fixed $p \in [0, \frac{1}{2})$ if additionally the $(\beta, B)$-Bernstein assumption holds for $(\mathcal{F}, P_{X, Y})$, then, with probability at least $1 - \delta$, it holds that
\begin{equation}
\label{bernst}
R^0(\widehat{f}_p) - R(f^*) \lesssim \left(\frac{Bd\log \frac{n}{d} + B\log \frac{1}{\delta}}{n}\right)^{\frac{1}{2 - \beta}}.
\end{equation}
Moreover, in that case we have, with probability at least $1 - \delta$,
\begin{equation}
\label{regionsize}
\Pr(\widehat{f}_p(X) = *) \lesssim \left(\frac{B^{\frac{2}{\beta}}d\log \frac{n}{d} + B^{\frac{2}{\beta}}\log \frac{1}{\delta}}{n}\right)^{\frac{\beta}{2 - \beta}}.
\end{equation}
\end{Proposition}

The proofs of both results of this section are  presented in Section \ref{sec:firstprop} and Section \ref{sec:secondprop}. The form of the bound \eqref{general} coincides with the optimal performance of ERM in the agnostic case (see \eqref{standardbound}). Similarly, \eqref{bernst} coincides with the ERM upper bound under the Bernstein assumption. We refer to Appendix \ref{app:ratiotype} or \cite[Equation $(22)$]{Hanneke16} for the corresponding ERM bounds.

\section{Fast Rates under Misspecification}
\label{sec:fasttartesmisspecificiation}
\subsection{Necessary and Sufficient Conditions}
\label{finitediametersection}

In this section, we extend simultaneously the results of Massart and N\'{e}d\'{e}lec \citep{Massart06} and  Ben David and Urner \cite{Bendavid14} and provide a necessary and sufficient condition for the fast rates under the misspecification. In contrast to the well-specified case, where ERM is almost optimal, it is crucial that one uses an improper procedure in the misspecified case. The necessary part will essentially follow from the two lower bounds shown in the aforementioned works. The sufficient part is our main algorithmic contribution which is obtained by a modification of our abstention Algorithm \ref{alg:abstentionalgorithm}. Before we proceed we need an additional definition.

\begin{Definition}[\cite{Bendavid14}]
\label{def:diameter}
The combinatorial diameter of $\mathcal F$ is defined as 
\[
D = \sup\limits_{f, g \in \mathcal F}\left|\{x \in \X: f(x) \neq g(x)\}\right|.
\]
\end{Definition}

When it is clear from the context we sometimes write \emph{diameter} instead of \emph{combinatorial diameter}. It is easy to see that the VC dimension is upper bounded by the diameter (i.e., $d \le D$). Although the combinatorial diameter is infinite for many natural classes, it is still relevant in the study of the learning rates. Indeed, almost all existing lower bounds in statistical learning are obtained by classes with a finite diameter \cite{Audibert09, Massart06}. Therefore, the classes with a finite diameter are expressive enough to show the tightness of the known upper bounds in different scenarios. One particular example of such a class is provided in Example \ref{exam:logfactor} below. Further, the importance of the combinatorial diameter comes from the lower bounds shown in \cite{Kaariainen05, Bendavid14}. Indeed, even if the margin parameter $h = 1$, but $D = \infty$ the lower bound  $\Omega\left(\sqrt{\frac{1}{n}}\right)$ holds for any learning algorithm. We discuss these lower bounds after Theorem \ref{thm:finitediameter}.

Let us introduce the only known upper bound for the classes with the finite combinatorial diameter. The result of \cite[Theorem 9]{Bendavid14} implies that in the case of $Y = f^*_B(X)$, that is $h = 1$, but $f^*_B$ is not necessary in the class $\F$, the following holds: there is a learning algorithm with its output classifier denoted by $\widehat{f}$ such that, with probability at least $1 - \delta$,
\begin{equation}
\label{bendavidineq}
R(\widehat{f}) - R(f^*) \lesssim \frac{D\log n + \log \frac{1}{\delta}}{n},
\end{equation}
where $D$ is the combinatorial diameter of $\F$. Extending this result to $h \le 1$ is quite problematic since the analysis of Ben-David and Urner is specific to the assumption that there is no noise in the labeling mechanism. Their algorithm actually memorizes the training sample and uses the predictions of a fixed function in $\F$ on all the instances that are not in the training sample. However, it is important that  their result is the first indication that one may get the fast rates for misspecified learning problems.

We are now ready to formulate the main algorithmic result of this paper. The proof of this result is presented in Section \ref{sec:proofofsecthm}.
\begin{Theorem}
\label{thm:finitediameter}
Consider a VC class $\mathcal{F}$ with the finite diameter $D$. If Massart's margin parameter $h$ is greater than zero and assumed to be known, then there is a learning algorithm (namely, Algorithm \ref{alg:finitediameter}) with its output classifier denoted by $\widehat{f}$ such that, with probability at least $1 - \delta$,
\[
R(\widehat{f}) - R(f^*) \lesssim  \frac{d\log\frac{n}{d} + D + \log \frac{1}{\delta}}{nh}.
\]
\end{Theorem}

Before discussing this result let us formulate the matching lower bounds. First, we refer to \citep[Theorem 4]{Massart06} which shows that if $d \ge 2$ and $h \ge \sqrt{\frac{d}{n}}$ then there is an in-expectation lower bound which holds for any classifier $\widetilde{f}$ based on the sample of size $n$. It can be written as follows
\begin{equation}
\label{eq:firstlower}
    \E R(\widetilde{f}) - R(f^*) \gtrsim \frac{d}{nh}.
\end{equation}
The distributions used in the proof of the lower bound \eqref{eq:firstlower} satisfy $f^*_B \in \F$.
Second, for $h = 1$, and $D = \infty$ the lower bound in \cite[Theorem 3]{Kaariainen05} holds for any classifier $\widetilde{f}$ and scales as
\begin{equation}
\label{eq:secondlower}
R(\widetilde{f}) - R(f^*) \gtrsim \sqrt{\frac{\log \frac{1}{\delta}}{n}}.
\end{equation}
We remark that a similar lower bound is also given in \cite[Theorem 4]{Bendavid14}. For the precise formulations of \eqref{eq:firstlower} and \eqref{eq:secondlower} we refer to the original papers. In Example \ref{exam:logfactor} we also provide a particular VC class $\F$ with $D = 2d$ such that the bound of Theorem \ref{thm:finitediameter} is tight.

Let us discuss some consequences of this result. Observe that Massart's margin condition restricts the conditional distribution $Y|X$ but does not affect the marginal distribution $P_X$ of $X$. Therefore, in view of the lower bounds \eqref{eq:firstlower} and \eqref{eq:secondlower} the result of Theorem \ref{thm:finitediameter} implies that for a given VC class $\F$ the learning rate $O\left(\frac{1}{n}\right)$ is possible with respect to \emph{any marginal distribution} $P_X$ \emph{if and only if} the combinatorial diameter $D$ is finite and Massart's margin parameter $h$ is non-zero (assuming that $h$ is a fixed parameter that does not depend on $n$).

Finally, we are ready to present Algorithm \ref{thm:finitediameter}. Let us highlight some intuition behind it. The proof of Theorem \ref{mainthmproc} suggests that the reason for the slow rates is a pair of almost ERMs such that their disagreement set is relatively large. However, under the assumption that $D < \infty$ we can actually use the labeled instances to specify our predictions on these problematic instances. 

For the sake of simplicity, we assume that the noise parameter $h$ is known to the learner. We note that the analysis may be generalized to the case where the precise value of $h$ is not known. In particular, the analysis holds whenever $h$ is only a lower on the margin parameter. We need the following definition: given a sample $S = \{(X_i, Y_i)_{i = 1}^n\}$ define the majority vote function $\text{maj}_{S}$ as follows
\begin{align*}
\text{maj}_{S}(x) &= \ind\left[\sum\limits_{i = 1}^n\ind[Y_i = 1\ \text{and}\ X_i = x] > \sum\limits_{i = 1}^n\ind[Y_i = 0\ \text{and}\ X_i = x]\right].
\end{align*}
This function counts the number of ones corresponding to $x$ and the number of zeros corresponding to $x$ and outputs the value according to the majority. If there is a tie, which may happen, in particular, if the instance $x$ does not belong to $X_1, \ldots, X_n$, then $\text{maj}_{S}$ outputs zero. 

\begin{framed}
\begin{algorithm}[from Theorem \ref{thm:finitediameter}]
\label{alg:finitediameter}
Given a VC class $\F$, a sample of size $3n$ from a distribution with known margin parameter $h$, proceed as follows
\begin{itemize}
    \item Apply Algorithm \ref{alg:abstentionalgorithm} to the first $2n$ instances with parameter $p=h/2$ to obtain $\widehat{f}_{h/2}$.
    \item Use the third part of the sample $S_3 = \{(X_i, Y_i)_{i = 2n+1}^{3n}\}$ to convert the $\{0,1, *\}$-valued function $\widehat{f}_{h/2}$ to the $\{0, 1\}$-valued function $\widehat{f}$ using the following rule:
\[
\widehat{f}(x) = \begin{cases} \widehat{f}_{h/2}(x), & \mbox{if } \widehat{f}_{h/2}(x) \in \{0, 1\}, \\ \text{\emph{maj}}_{S_3}(x), & \mbox{if } \widehat{f}_{h/2}(x) = *. \end{cases}
\]
\item Output $\widehat{f}$.
\end{itemize}
\end{algorithm}
\end{framed} 

\begin{Remark}
Observe that if $h = 1$ Algorithm \ref{alg:finitediameter} gives a bound better than \eqref{bendavidineq}. Moreover, our algorithm does not require to remember the entire training sample, which is the case in \eqref{bendavidineq}. 
\end{Remark}

\subsection{A Distribution Dependent Upper Bound}
\label{sec:distribdepdiameter}
In this section, we present a corollary of Theorem \ref{thm:finitediameter} and provide a distribution-dependent version of the combinatorial diameter. In order to simplify the presentation of this more involved case, we make two assumptions:
\begin{itemize}
    \item We consider the case of the deterministic labeling, that is $h = 1$,
    \item We assume that the marginal distribution $P_X$ is known to the learner.
\end{itemize}

As already mentioned, if the diameter $D$ is infinite, then it is impossible to improve the $\Omega(\frac{1}{\sqrt{n}})$ lower bound even in the case of $h = 1$. However, this lower bound is demonstrated by the worst-case marginal distribution $P_X$. Therefore, in order to generalize our bound, we need to introduce a $P_X$-dependent analog of the combinatorial diameter. 

The proof of Theorem \ref{thm:finitediameter} is based on the following fact: for any two functions $f, g \in \F$ the boolean cube of functions supported on the set $\{x \in \X: f(x) \neq g(x)\}$ is finite. This implies, in particular, that the corresponding boolean cube is learnable with respect to any distribution. This fact is explicitly used in the proof of Theorem \ref{thm:finitediameter}.
This is, of course, the case since we consider the class $\F$ with finite combinatorial diameter $D$. However, in order to make the same argument applicable for some specific distributions $P_X$ in the case of $D = \infty$, we need some additional steps. We introduce the following notation. For $\mathcal H \subseteq \{0, 1\}^{\mathcal X}$ and $\varepsilon > 0$ let $\mathcal{N}(\mathcal H, \varepsilon,  L_1)$ denote the covering number (see e.g., \cite{Boucheron13}) with respect to $L_1(P_X)$ metric. Denote the corresponding covering set by $N(\mathcal H, \varepsilon,  L_1)$. This means, in particular, that $|N(\mathcal H, \varepsilon,  L_1)| = \mathcal{N}(\mathcal H, \varepsilon,  L_1)$. For a pair of binary functions $f, g$ denote by $\mathcal{C}_{f, g}$ the set of binary classifiers that shatters the set $\mathcal X^{\prime} = \{x \in \X: f(x) \neq g(x)\}$. These functions are set to be equal to zero on $\X \setminus \mathcal X^{\prime}$. The simplest distribution-dependent analog of the combinatorial diameter we found in our context is the following fixed point.
\begin{Definition}[A $P_X$-dependent analog of the combinatorial diameter]
Let $c_1 > 0$ be a numerical constant to be specified in our proofs. Given a class of binary functions $\F$, a marginal distribution $P_X$ and a sample size $n$, we define
\[
D_{P_X}(n) := n\ \sup\limits_{f, g \in \F}\max\limits\left\{\gamma \ge 0:   c_1 n\gamma \le \log_2 \mathcal{N}(\mathcal{C}_{f, g}, \gamma, L_{1})\right\}.
\]
\end{Definition}
\begin{Remark}
It is easy to show that $D_{P_X}(n) \le \frac{D}{c_1}$, since for any $\gamma > 0$,
\[
\sup\limits_{f, g \in \F}\log_2 \mathcal{N}(\mathcal{C}_{f, g}, \gamma, L_{1}) \le D.
\]
However, for some classes $\F$ and some marginal distributions $P_{X}$ the value of $D_{P_X}(n)$ can be finite even if $D$ is infinite. The simplest example is the situation where for any $\varepsilon \in [0, 1]$ and any two $f, g \in \F$ there is a finite set $\X_{f, g}^{\prime}(\varepsilon) \subseteq \{x \in \X: f(x) \neq g(x)\}$ such that
\[
\Pr\left(\{x \in \X: f(x) \neq g(x)\}\setminus \X^{\prime}(\varepsilon)\right) \le \varepsilon.
\]
Since the binary cube supported on $\X_{f, g}^{\prime}(\varepsilon)$ forms an $\varepsilon$-cover of the set $\mathcal{C}_{f, g}$, we have
\[
D_{P_X}(n) \le n\ \sup\limits_{f, g \in \F}\max\limits\left\{\gamma \ge 0:   c_1 n\gamma \le \left|\X^{\prime}_{f, g}(\gamma)\right|\right\}.
\]
\end{Remark}
The main result of this paragraph is the following analog of Theorem \ref{thm:finitediameter}.
\begin{Corollary}
\label{finitediametercorollary}
If Massart's margin assumption \eqref{eq:massartsmargin} is satisfied with $h = 1$, then there is a $P_X$-dependent learning algorithm (namely, Algorithm \ref{alg:finitediameterdistibution}) with its output classifier denoted by $\widehat{f}$ such that, with probability at least $1 - \delta$,
\[
R(\widehat{f}) - R(f^*) \lesssim  \frac{d\log\frac{n}{d} + D_{P_X}(n) + \log \frac{1}{\delta}}{n}.
\]
\end{Corollary}

We present the corresponding learning algorithm.

\begin{framed}
\begin{algorithm}
\label{alg:finitediameterdistibution}
Given a VC class $\F$, a sample of size $3n$ from a distribution with known margin parameter $h$, proceed as follows
\begin{itemize}
    \item Apply Algorithm \ref{alg:abstentionalgorithm} to the first $2n$ instances with $p=1/2$ to construct $\widehat{f}_{1/2}$.
    \item Use the third part of the sample $S_3 = \{(X_i, Y_i)_{i = 2n+1}^{3n}\}$ to convert the $\{0,1, *\}$-valued function $\widehat{f}_{h/2}$ to the $\{0, 1\}$-valued function $\widehat{f}$. To do so, recall that $\widehat{g}$, is the ERM based on the first part of the sample and let $\widehat{g}_1 \in \F$ be the classifier that together with $\widehat g$ defines $\widehat{f}_{1/2}$. For $c_2 > 0$, which is tuned in the proof, consider the value
    \begin{equation}
    \label{radius}
    r =  c_2\left(\frac{D_{P_X}(n)}{n} + \frac{\log \frac{1}{\delta}}{n}\right).
    \end{equation}
    Given the ($S_1, S_2$-dependent) set $\mathcal{C}^{\prime} = N(\mathcal{C}_{\widehat{g},\widehat{g}_1 },r , L_{1})$, define
    \begin{equation}
    \label{ermovernet}
    \widetilde{g} = \argmin\limits_{g \in \mathcal{C}^{\prime}}\sum\limits_{i = 2n + 1}^{3n}\ind\left[g(X_i) \neq Y_i\ \text{and} \ \widehat{f}_{1/2}(X_i) = *\right].
    \end{equation}
    Finally, set
    \[
    \widehat{f}(x) = \begin{cases} \widehat{f}_{1/2}(x), & \mbox{if } \widehat{f}_{1/2}(x) \in \{0, 1\}, \\ \widetilde{g}(x), & \mbox{if } \widehat{f}_{1/2}(x) = *. \end{cases}
    \]
\item Output $\widehat{f}$.
\end{itemize}
\end{algorithm}
\end{framed} 
Finally, we note that a generalization of this result for $h \neq 1$ can be done using the techniques in \cite{Zhivotovskiy17b}. We avoid these technical generalizations to make the presentation more transparent.

\subsection{Removing the Logarithmic Factor for \texorpdfstring{$h = 1$}{Lg}}
\label{app:log}
The possibility of refining the logarithmic factors in some related bounds  has taken some attention recently \cite{Hanneke16, Zhivotovskiy18, Zhivotovskiy17, Bousquet20}. We start with the following motivating example which shows (using the known lower bounds) that the logarithmic factor in the bound of Theorem \ref{thm:finitediameter} can be sometimes necessary if the margin parameter $h < 1$. 
\begin{Example}
\label{exam:logfactor}
Fix $d \ge 2$ and consider the class $\F_d$ consisting of all binary functions taking the value one on at most $d$ instances in $\X$. This class satisfies $D = 2d$ whenever $|\X| \ge 2d$. Our Theorem \ref{thm:finitediameter} implies that without assuming that $f^*_B \in \F_d$, a classifier $\widehat{f}$ constructed by Algorithm \ref{alg:finitediameter} satisfies, with probability at least $1 - \delta$,
\begin{equation}
\label{eq:upperboundfd}
R(\widehat{f}) - \inf\limits_{f \in \F_d}R(f) \lesssim  \frac{d\log\frac{n}{d}+ \log \frac{1}{\delta}}{nh}.
\end{equation}
However, the results in \citep[Theorem 4 and Theorem 5]{Massart06} imply that provided that $h \ge \sqrt{\frac{d}{n}}$ for any classifier $\widetilde{f}$ there is a distribution with margin parameter $h$ such that
\begin{equation}
\label{eq:lowerboundwithlog}
\E R(\widetilde{f}) - \inf\limits_{f \in \F_d}R(f) \gtrsim \frac{d}{nh} + (1 - h)\frac{d\log \frac{nh^2}{d}}{nh}.
\end{equation}
\end{Example}
The important part of the lower \eqref{eq:lowerboundwithlog} is that the logarithmic factor disappears when $h = 1$ whereas it is still present in \eqref{eq:upperboundfd}. In this section we present a particularly simple upper bound that removes the superfluous logarithmic factor for $h = 1$. Thus, we are improving the in-expectation version of the upper bound of Theorem \eqref{thm:finitediameter} as well as the bound in \cite[Theorem 9]{Bendavid14}. The form of the result is similar to the upper bound for the one inclusion graph algorithm in \cite{Haussler94}, but in our case, $D$ plays the role of the VC dimension.
\begin{Proposition}
\label{prop:logfactor}
Consider a VC class $\mathcal{F}$ with the finite combinatorial diameter $D$. If Massart's margin assumption \eqref{eq:massartsmargin} holds with $h = 1$, then there is a learning algorithm with its output classifier denoted by $\widehat{f}$ such that, with probability at least $1 - \delta$,
\[
\E R(\widehat{f}) - R(f^*) \le \frac{D}{n + 1}.
\]
\end{Proposition}
The proof of this fact is presented in Section \ref{sec:logfactorproof}.

\section{Proof of Theorem \ref{mainthmproc}}
\label{sec:proofofabstthm}
We start with the following key observation. Let $\G$ be a class of $\{0, 1, \frac{1}{2}\}$-valued functions.
For $q \ge 1$ and $g \in \G$ define the $\ell_q$ risk as
\[
\ell_q(g) = \E|g(X) - Y|^q.
\]
Observe that if the function $g$ is only $\{0, 1\}$-valued then $\ell_q(g) = R(g)$, that it, the $\ell_q$ risk is equal to the binary risk for the binary-valued functions. In particular, the optimal binary function $f^*$ is also minimizing the $\ell_q$ risk among the corresponding binary functions.  Observe that since $Y \in \{0, 1\}$ we have for any $g \in \G$,
\[
\ell_q(g) = \Pr\left(g(X) \neq Y\ \text{and}\ g(X) \in \{0, 1\}\right) + \frac{1}{2^q}\Pr\left(g(X) = \frac{1}{2}\right).
\]
Usnig the definition \eqref{rp} of the $R^p$ risk we conclude in the notation of Theorem \ref{mainthmproc} that
\begin{equation}
\label{eq:excessriskrelation}
R^p(\widehat{f}_p) - R(f^*) = \ell_q(\widehat{f}_p^{\prime}) - \ell_q(f^*),
\end{equation}
where $\widehat{f}_p^{\prime}$ is obtained from $\widehat{f}_p$ by replacing $*$ by $\frac{1}{2}$ and $\frac{1}{2} - p = \frac{1}{2^q}$. 

Therefore, for the rest of the proof we may restrict our analysis on the $\ell_q$ risk and $\{0, 1, \frac{1}{2}\}$-valued classifiers. Finally, in Section \ref{sec:coorectqviap} we choose an appropriate value of $q$ based on $p$. Throughout the proofs we will be taking the union bounds over a finite number of events without mentioning it. In all the cases this only affects the constant factors which are always absorbed by the symbol $\lesssim$.

\subsection{Some Results from Empirical Process Theory}
\label{empprocesssection}
We present several results related to the aggregation of classifiers under the $\ell_q$ risk for $q \ge 1$. In this section we sometimes assume that $\mathcal F$ is a class of discrete-valued functions taking their values in $[0, 1]$ and we keep our convention on the definition of the growth function. Similarly, depending on the context we may consider $Y \in [0, 1]$. We define the \emph{excess loss class} as
\begin{equation}
\label{lq}
\mathcal{L}_q = \{(x, y) \to |f(x) - y|^q - |f^*(x) - y|^q: f \in \F\},
\end{equation}
where $f^* = \argmin\limits_{f \in \mathcal F}\E|f(X) - Y|^q$. Recall that as long as only $\{0, 1\}$-valued functions and values of $Y$ are considered, the $\ell_q$ risk is equivalent to the binary risk. We start with the following standard lemma. 
\begin{Lemma}
\label{pairwise}
Consider a class of binary-valued functions $\mathcal F$ of VC dimension $d$ and assume that $Y \in \{0, 1\}$. Fix $\delta\in(0,1)$ and $\alpha = \sqrt{\frac{d\log \frac{n}{d} + \log(1/\delta)}{n}}$. Then, simultaneously for all $f, g \in \F$, with probability at least $1 - \delta$,
\begin{equation}
\label{pairwisedist}
\bigl|P_n|f - g| - P|f - g|\bigr| \lesssim \alpha\sqrt{P_n|f - g|} + \alpha^2,
\end{equation}
as well as
\begin{equation}
\label{pairwisedistsec}
\bigl|P_n|f - g| - P|f - g|\bigr| \lesssim \alpha\sqrt{P|f - g|} + \alpha^2.
\end{equation}
Fix $q \ge 1$. Simultaneously for all $h \in \mathcal{L}_q$ and the corresponding $f \in \F$, with probability at least $1 - \delta$,
\begin{equation}
\label{starshaped}
|P h - P_n h| \lesssim \alpha\sqrt{P_n|f - f^*|} + \alpha^2.
\end{equation}
\end{Lemma}

Because of the general importance of Lemma \ref{pairwise}, we give some new insights on its proof and applications in Appendix \ref{app:ratiotype}. In particular, we show that it is possible to prove the optimal risk bounds under Massart's and Tsybakov's \cite{Massart06, Tsybakov04} noise conditions using only on the original ratio-type estimates of Vapnik and Chervonenkis \cite{Vapnik74}. 
For a class $\mathcal{G}$ and $s \ge 1$ we define the $L_s$ diameter as 
\begin{equation}
\label{eq:diam}
\D(\mathcal{G}, L_s) = \sup\limits_{f, g \in \mathcal G}\|f - g\|_{L_s}.
\end{equation}
The next result is also standard. We provide a detailed sketch of its proof for the sake of completeness.
\begin{Lemma}
\label{contractionconcentration}
Consider a class of discrete-valued functions $\mathcal F$ taking their values in $[0, 1]$ and assume that $Y \in [0, 1]$ almost surely, $q \in [1, 2]$ and that $\F$ has a finite growth function. Then for any $\delta\in(0,1)$, with probability at least $1 - \delta$,
\[
\sup\limits_{h \in \mathcal{L}_q}|Ph - P_n h| \lesssim  \D(\mathcal F, L_{2})\sqrt{\frac{\log \mathcal{S}_{\mathcal{F}}(n) + \log(1/\delta)}{n}} + \frac{\log \mathcal{S}_{\mathcal{F}}(n) + \log(1/\delta)}{n}.
\]
\end{Lemma}

\begin{proof}
The strategy of the proof is quite standard (see e.g.,  \cite[Lemma 3.2]{Lecue09}, where the analysis is presented in the case of finite classes). At first, we provide an upper bound on $\E\sup\limits_{h \in \mathcal{L}_q}|Ph - P_n h|$. Observe that since $q \in [1, 2]$ and due to our boundedness assumption we have for any $f, g \in \mathcal F$,
\begin{equation}
\label{lip}
||f(X) - Y|^q - |g(X) - Y|^q| \le 2|f(X) - g(X)|.
\end{equation}
Therefore, by the symmetrization argument \cite{Talagrand14} and the standard bound for Rademacher averages \cite[Theorem 3.3 and inequality (6)]{Boucheron05b} together with Jensen's inequality
\begin{align*}
\E\sup\limits_{h \in \mathcal{L}_q}|Ph - P_n h| &\lesssim \E\sup\limits_{h \in \mathcal{L}_q}\sqrt{\frac{\log \mathcal{S}_{\mathcal{F}}(n)P_n h^2}{n}}
\\
&\lesssim \sqrt{\frac{\log \mathcal{S}_{\mathcal{F}}(n)}{n}\E\sup\limits_{f \in \mathcal F} P_n(f - f^*)^2},
\end{align*}
where we used the fact that $\mathcal{L}_q$ has the same growth function as $\mathcal F$ together with \eqref{lip}. Now, using the same lines
\begin{align*}
\E\sup\limits_{f \in \mathcal F} P_n(f - f^*)^2 &\le \E\sup\limits_{f \in \mathcal F} |P_n(f - f^*)^2 - P(f - f^*)^2| + \sup\limits_{f \in \mathcal{F}}\E(f - f^*)^2
\\
&\lesssim \sqrt{\frac{\log \mathcal{S}_{\mathcal{F}}(n)}{n}}+ \sup\limits_{f \in \mathcal{F}}\E(f - f^*)^2.
\end{align*}
Combining two bounds together we have
\[
\E\sup\limits_{h \in \mathcal{L}_q}|Ph - P_n h| \lesssim  \D(\mathcal F, L_{2})\sqrt{\frac{\log \mathcal{S}_{\mathcal{F}}(n)}{n}} + \frac{\log \mathcal{S}_{\mathcal{F}}(n)}{n}.
\]
To prove the high probability version of the bound we apply Talagrand's concentration inequality for empirical processes \cite[Theorem 12.2]{Boucheron13} to the process $|Ph - P_n h|$ indexed by $\mathcal{L}_q$ and combine it with \eqref{lip}. The claim follows.
\end{proof}

\subsection{Pairwise Aggregation}
Finally, we provide the aggregation algorithm for the $\ell_q$ risk. The algorithm of interest is a simplified version of the empirical star algorithm \cite{Audibert07} and the two-step aggregation algorithm of \cite{Lecue09} adapted for our purposes. Unfortunately, to the best of our knowledge, none of the known aggregation results can be directly applied in our case, since they are usually tuned to finite dictionaries, some of them are specific to the squared loss or some special moment assumptions, and the class of output functions $\mathcal{G}$ is almost always too large. Thus, we provide a simple self-contained analysis. The aggregating algorithm we present now is essentially the $\ell_q$ risk version of our abstention Algorithm \ref{alg:abstentionalgorithm}.

\begin{framed}
\begin{algorithm}[An aggregation algorithm for VC classes with the $\ell_q$ risk]
\label{alg:aggregationalgorithm}
Let $c > 0$ be a specifically chosen numerical constant in Definition \ref{def:aerm}. Given a class $\F$ of binary functions of VC dimension $d$, a sample of size $2n$, a confidence level $\delta$, and $q\in[0,\frac{1}{2}]$, proceed as follows:
\begin{itemize}
    \item Based on the first half of the sample $\{(X_i, Y_i)_{i = 1}^n\}$ denote ERM with respect to the binary loss by $\widehat{g}$.
    \item Based on $\{(X_i, Y_i)_{i = 1}^{n}\}$ construct the class of almost empirical minimizers $\widehat{\F}$ as in Definition \ref{def:aerm}.
    \item Consider the class $\widehat{\mathcal{G}} = \left\{\frac{f + \widehat{g}}{2}: f \in \widehat{\mathcal F}\right\}$ of $\{0,1, \frac{1}{2}\}$-valued functions.
    \item Output $\widehat{f}$ defined as minimizer of the empirical $\ell_q$ risk over the class $\widehat{\mathcal{G}}$ based on the second half of the sample $\{(X_i, Y_i)_{i = n + 1}^{2n}\}$.   
\end{itemize}
\end{algorithm}
\end{framed} 

Before analyzing this algorithm we need one more technical result.
\begin{Lemma}
\label{locoftarget}
There is a choice of the constant $c$ in Definition \ref{def:aerm} such that, with probability at least $1 - \delta$, it holds that
\[
f^* \in \widehat \F,
\]
and that for any $f \in \widehat \F$, with probability at least $1 - \delta$,
\[
R(f) - R(f^*) \lesssim \alpha \D(\widehat{\mathcal F}, L_{2}) + \alpha^2.
\]
\end{Lemma}
\begin{proof}
First, we prove that $f^* \in \widehat \F$ with high probability. Denote $h_{\widehat{g}}=|\widehat{g}-Y|-|f^*-Y|$. Since $R(f^*)\le R(\widehat{g})$, we have
\[
R_n(f^*)-R_n(\widehat{g}) = R(f^*)-R(\widehat{g}) + Ph_{\widehat{g}}-P_nh_{\widehat{g}} \le |Ph_{\widehat{g}}-P_nh_{\widehat{g}}|.
\]
Therefore, by Lemma \ref{pairwise}, we have with probability $1-\delta$,
\[
R_n(f^*)-R_n(\widehat{g}) \le c_1\paren*{\alpha\sqrt{P_n|\widehat{g}-f^*|}+\alpha^2}\,,
\]
for some $c_1 > 0$. By choosing $c = c_1$ in Definition \ref{def:aerm} we show that, with probability $1-\delta$, it holds that $f^* \in \widehat \F$. Further, observe that using $2\sqrt{ab} \le a + b$, for $a, b \ge 0$ the inequality \eqref{pairwisedistsec} implies that, with probability at least $1 - \delta$, simultaneously for all $f, g \in \mathcal F$,
\begin{equation}
\label{distancerelation}
P_n|f - g| \lesssim P|f - g| + \alpha^2.
\end{equation}
From now on we work on the event $f^* \in \widehat{\F}$. For any $f \in \widehat{\mathcal F}$ and $q \ge 1$, with probability at least $1 - \delta$,
\begin{align*}
R(f) - R(f^*) &\le \sup\limits_{h \in \mathcal{L}_q(\widehat \F)}|Ph - P_n h|+ R_n(f)  - R_n(f^*)
\\
&\le \sup\limits_{h \in  \mathcal{L}_q(\widehat \F)}|Ph - P_n h|+ R_n(f)  - R_n(\widehat{g}) \,\, \mbox{(since $R_n(\widehat{g}) \le R_n(f^*)$)}
\\
&\lesssim \sup\limits_{h \in  \mathcal{L}_q(\widehat \F)}|Ph - P_n h|+ \alpha\sqrt{P_n|\widehat{g} - f|} + \alpha^2 \,\, \mbox{(since $f \in \widehat{\F}$)}
\\
&\lesssim \sup_{\tilde{f}\in\widehat{\F}}\alpha\sqrt{P|\tilde{f}-f^*|} + \alpha\sqrt{P|\widehat{g} - f|} + \alpha^2  \,\, \mbox{(by Lemma \ref{pairwise} and \eqref{distancerelation})}
\\
&\lesssim \alpha \D(\widehat{\mathcal F}, L_{2}) + \alpha^2,
\end{align*}
where in the last line we used that $f^* \in \widehat{\F}$.
\end{proof}

In view of the equality \eqref{eq:excessriskrelation} we are going to prove Theorem \ref{mainthmproc} by showing the following lemma.
\begin{Lemma}
\label{aggtheorem}
Consider $q \in (1, 2]$, and assume that $Y \in \{0, 1\}$ and that $\mathcal F$ is a binary class with VC dimension $d$. Then, for any $\delta\in(0,1)$, we have for the output $\widehat{f}$ of Algorithm \ref{alg:aggregationalgorithm}, with probability at least $1 - \delta$,
\begin{equation}
\label{excessrisk}
\ell_q(\widehat{f}) - \ell_q(f^*) \lesssim \frac{d\log \frac{n}{d} + \log \frac{1}{\delta}}{n(q - 1)},
\end{equation}
where $f^*$ minimizes the binary risk $R(f)$ in $\mathcal F$.
\end{Lemma}

\begin{proof}
To analyze the properties of the $\ell_q$ risk we need some basic properties of the function $|x|^q$ for $q \in (1, 2]$ and $|x| \le 1$. By Taylor's formula, for $x, y \neq 0$ we have
\[
|y|^q - |x|^q = q|x|^{q - 1}(y - x) + \frac{q(q-1)|\xi|^{q - 2}}{2}(y - x)^2,
\]
where $\xi$ is some mid-point. Since $q \le 2$, we have $|\xi|^{q - 2} \ge 1$, therefore
\[
|y|^q - |x|^q \ge q|x|^{q - 1}(y - x) + \frac{q(q-1)}{2}(y - x)^2,
\]
The parameter $q(q - 1)$ plays a role of the modulus of strong convexity in this case. In particular, we have in our range (see e.g., \cite{Boyd98}) that for any $t \in [0, 1]$,
\begin{equation}
\label{strconvexineq}
|tx + (1 - t)y|^q \le t|x|^q + (1 - t)|y|^q - \frac{1}{2}q(q - 1)t(1 - t)(x - y)^2.
\end{equation}

We note that the logarithm of the growth function of $\widehat{\mathcal G}$ is controlled up to an absolute constant by the logarithm of the growth function of $\F$. Moreover, the $L_2$-diameter of the random set $\widehat{\mathcal G}$ is the same as the diameter of $\widehat{\mathcal F}$. Assume that $g^*$ minimizes $\ell_q(g)$ over $g \in \widehat{\mathcal{G}}$. Since we work with the independent second part of the sample, we may assume that $g^*$ is fixed. Then for $g \in \widehat{\mathcal G}$, defined by $g = (\widehat{g} + h)/2$, where $h \in \widehat{\F}$ is a function with $\|\widehat{g} - h\|_{L_2} \ge  \D(\widehat{\mathcal F}, L_2)/2$ (it exists by the definition), we have
\begin{align*}
\ell_q(g^*)&\le \ell_q(g) \le \frac{1}{2}P|Y - \widehat{g}|^q + \frac{1}{2}P|Y - h|^q - \frac{1}{8}(q - 1)P(\widehat{g} - h)^2 \,\, \mbox{(by \eqref{strconvexineq})}
\\
&=\frac{1}{2}P|Y - \widehat{g}| + \frac{1}{2}P|Y - h| - \frac{1}{8}(q - 1)P|\widehat{g} - h|
\\
&\le P|Y - f^*|^q + 2c( \D(\widehat{\mathcal F}, L_{2})\alpha + \alpha^2) - \frac{1}{32}(q - 1) \D(\widehat{\mathcal F}, L_2)^2 \,\, \mbox{(by Lemma \ref{locoftarget})},
\end{align*}
where in the last line we also used that $\widehat{g} \in \widehat{\mathcal{F}}$.
Finally, by Lemma \eqref{contractionconcentration} and the fact that $\widehat{f}$ minimizes the empirical risk we have for some $c_1 > 0$,
\[
\ell_q(\widehat{f}) - \ell_q(g^*) \le c_1\left( \D(\widehat{\mathcal F}, L_{2})\alpha + \alpha^2\right).
\]
This implies for some $c_2 > 0$,
\begin{align*}
\ell_q(\widehat{f}) - \ell_q(f^*) &\le c_2\left( \D(\widehat{\mathcal F}, L_{2})\alpha + \alpha^2\right) - \frac{1}{32}(q - 1) \D(\widehat{\mathcal F}, L_2)^2 \lesssim \frac{d \log \frac{n}{d}+ \log \frac{1}{\delta}}{(q - 1)n},
\end{align*}
where we used that for any $x$, it holds that $c_2x- \frac{1}{32}(q - 1)x^2 \le 8c_2^2/(q - 1)$. The claim follows. 
\end{proof}

\subsection{From Aggregation to Chow's Reject Option Model}
\label{sec:coorectqviap}
It is only left to choose the correct value of $q$ depending on $p$ in Lemma \ref{aggtheorem}.
We need some elementary inequalities. 

{\bf Case $p \in [0, \frac{1}{4}]$}. We define $\frac{1}{2} - p = \frac{1}{2^q}$. Since we consider $q \in [1, 2]$, our analysis will cover only the case $p \in [0, \frac{1}{4}]$. An elementary inequality for $\frac{1}{2^q} = \frac{1}{2}\exp(-(q - 1)\log 2)$ implies,
\[
\frac{1}{2} - \frac{1}{2}(q - 1)\log 2 \le \frac{1}{2^q} \le \frac{1}{2} - \frac{1}{4}(q - 1)\log 2,
\]
provided that $q \in [1, 2]$. Therefore, we have $\frac{1}{q - 1} \le \frac{\log 2}{2p}$. Now using the result of Lemma \ref{aggtheorem} we have, with probability at least $1 - \delta$,
\[
R^p(\widehat{f}_p) - R(f^*) \lesssim \frac{d\log\frac{n}{d} + \log \frac{1}{\delta}}{np}.
\]
where $\widehat{f}_p$ is defined as $\widehat{f}$ in Theorem \ref{aggtheorem} by replacing the outputs $\frac{1}{2}$ with $*$. 

{\bf Case $p \in (\frac{1}{4}, \frac{1}{2}]$}. Observe that the bound of Theorem \ref{mainthmproc} changes by at most the factor of two if $p \in [\frac{1}{4}, \frac{1}{2}]$. Therefore, to prove the result, we may redefine $\widehat{f}_p$ in a way such that if $p \in (\frac{1}{4}, \frac{1}{2}]$ the output function is $\widehat{f}_{1/4}$.
\qed
\section{Remaining Proofs}
\label{sec:remainingproofs}
\subsection{Proof of Proposition \ref{prop:mainthmprocermfirst}}
\label{sec:firstprop}
We prove the statement for an arbitrary function in the class $\widehat{\G}$ defined in Algorithm \ref{alg:abstentionalgorithm}. Therefore, it does not matter which function in $\widehat{\G}$ will be chosen by $\widehat{f}_p$. We have
\[
\widehat{\mathcal F} \subseteq \{f \in \F: R_n(f) - R_n(\widehat{g}) \le 2c\alpha\},
\]
provided that $\alpha \le 1$. From the standard uniform convergence result \cite{Vapnik74} we have, with probability at least $1 - \delta$,
\[
R(f) - R(\widehat{g}) - R_n(f) + R_n(\widehat{g}) \lesssim \sqrt{\frac{d\log \frac{n}{d} + \log \frac{1}{\delta}}{n}}.
\]
This implies that for any $f \in \widehat{\mathcal F}$, with probability at least $1 - \delta$,
\begin{equation}
\label{riskbounds}
R(f) - R(\widehat{g}) \lesssim \sqrt{\frac{d\log \frac{n}{d} + \log \frac{1}{\delta}}{n}} \quad\text{and} \quad R(\widehat{g}) - R(f^*) \lesssim \sqrt{\frac{d\log \frac{n}{d} + \log \frac{1}{\delta}}{n}}.
\end{equation}
Observe that for any $f \in \widehat{\mathcal F}$,
\[
R(f) + R(\widehat{g}) = \E(\ind[f(X) \neq Y] +  \ind[\widehat{g}(X) \neq Y]) = 2R^{0}(f^{\prime}),
\]
where $f^{\prime}$ is defined by
\begin{equation}
\label{fprime}
f^{\prime}(x) = \begin{cases} f(x), & \mbox{if } f(x) = \widehat{g}(x), \\ *, & \mbox{if } f(x) \neq \widehat{g}(x). \end{cases}
\end{equation}
At the same time, using \eqref{riskbounds} we have, with probability at least $1 - \delta$,
\[
R^{0}(f^\prime) - R(f^*) = \frac{1}{2}(R(f) - R(\widehat{g})) + R(\widehat{g})- R(f^*) \lesssim \sqrt{\frac{d\log \frac{n}{d} + \log \frac{1}{\delta}}{n}}.
\]
Since the construction of $\widehat{\mathcal{G}}$ does not depend on $p$ and $f^\prime$ is an arbitrary function in $\widehat{\mathcal{G}}$, we prove the claim.
\qed
\subsection{Proof of Proposition \ref{prop:mainthmprocermsecond}}
\label{sec:secondprop}
At first, we prove that, with probability at least $1 - \delta$, for any $f \in \widehat{\F}$,
\begin{equation}
\label{inc}
R(f) - R(f^*) \lesssim \left(\frac{Bd\log \frac{n}{d} + B\log \frac{1}{\delta}}{n}\right)^{\frac{1}{2 - \beta}}.
\end{equation}
We first recall that $\alpha = \sqrt{\frac{d\log \frac{n}{d} + \log\frac{1}{\delta}}{n}}$. With total probability at least $1-\delta$, for any $f \in \widehat{\F}$ it holds that
\begin{align*}
R(f) - R(f^*) &\lesssim R_n(f) - R_n(\widehat{g}) + R_n(\widehat{g}) - R_n(f^*) + \alpha^2 + \alpha\sqrt{P|f - f^*|} \,\, \mbox{(by \eqref{starshaped})}
\\
&\lesssim R_n(f) - R_n(\widehat{g}) + \alpha^2 + \alpha\sqrt{P|f - f^*|} \,\, \mbox{(since $R_n(\widehat{g}) \le R_n(f^*)$)}
\\
&\lesssim  \alpha^2  + \alpha\sqrt{P_n|\widehat{g} - f|} + \alpha\sqrt{P|f - f^*|} \,\, \mbox{(since $f \in \widehat{\F}$)}
\\
&\lesssim  \alpha^2  + \alpha\sqrt{P|\widehat{g} - f|} + \alpha\sqrt{P|f - f^*|} \,\, \mbox{(by \eqref{distancerelation})}
\\
&\lesssim \alpha^2  + \alpha\sqrt{P|\widehat{g} - f^*|} + \alpha\sqrt{P|f - f^*|} \,\, \mbox{(by $P|\widehat{g} - f| \le P|\widehat{g} - f^*| + P|f - f^*|$)}
\\
&\lesssim \alpha^2  + \alpha B^{\frac{1}{2}}(R(\widehat{g}) - R(f^*))^{\frac{\beta}{2}} + \alpha B^{\frac{1}{2}}(R(f) - R(f^*))^{\frac{\beta}{2}},
\end{align*}
where in the last line we used the Bernstein assumption.
Since $\widehat{g}$ is ERM it follows from \cite{Massart06} (see also Appendix \ref{app:ratiotype} or \cite[Equation $22$]{Hanneke16}) that, with probability at least $1 - \delta$,
\begin{equation}
\label{eq:exriskerm}
R(\widehat{g}) - R(f^*) \lesssim \left(\frac{Bd\log \frac{n}{d} + B\log \frac{1}{\delta}}{n}\right)^{\frac{1}{2 - \beta}}.
\end{equation}
This implies,
\begin{align*}
R(f) - R(f^*) &\lesssim \alpha^2 + \alpha B^{\frac{1}{2}}(R(f) - R(f^*))^{\frac{\beta}{2}} + \alpha B^{\frac{1}{2}}\left(\frac{Bd\log \frac{n}{d} + B\log \frac{1}{\delta}}{n}\right)^{\frac{\beta}{2(2 - \beta)}}
\\
&\lesssim \alpha^2 + \alpha B^{\frac{1}{2}}(R(f) - R(f^*))^{\frac{\beta}{2}} + \left(\frac{Bd\log \frac{n}{d} + B\log \frac{1}{\delta}}{n}\right)^{\frac{1}{2 - \beta}}.
\end{align*}
Solving this inequality with respect to $R(f) - R(f^*)$ we prove \eqref{inc}. Thus, for any $f \in \widehat{\mathcal F}$,
\begin{align*}
\frac{1}{2}(R(f) + R(\widehat{g})-2 R(f^*)) = R^{0}(f^{\prime}) - R(f^*),
\end{align*}
where $f^{\prime}$ is defined by \eqref{fprime}. At the same time,
\[
R^{0}(f^{\prime}) - R(f^*) = \frac{1}{2}(R(f) + R(\widehat{g})-2 R(f^*)) \lesssim \left(\frac{Bd\log \frac{n}{d} + B\log \frac{1}{\delta}}{n}\right)^{\frac{1}{2 - \beta}}.
\]
This proves the first part of the claim.
Finally, for any $f \in \widehat{\mathcal F}$ and $f^{\prime}$ given by \eqref{fprime} we have, with probability at least $1 - \delta$,
\begin{align*}
\Pr(f^{\prime}(X) = *) &= \Pr(f(X) \neq \widehat{g}(X)) 
\\
&\le \Pr(f(X) \neq f^*(X)) + \Pr(\widehat{g}(X) \neq f^*(X)) 
\\
&\le B(R(f) - R(f^*))^{\beta} + B(R(\widehat{g}) - R(f^*))^{\beta}  \,\, \mbox{(by the Bernstein assumption)}
\\
&\lesssim B\left(\frac{Bd\log \frac{n}{d} + B\log \frac{1}{\delta}}{n}\right)^{\frac{\beta}{2 - \beta}},
\end{align*}
where we used \eqref{inc} and \eqref{eq:exriskerm}. 
The claim follows. \qed

\subsection{Proof of Theorem \ref{thm:finitediameter}}
\label{sec:proofofsecthm}

First, recall the standard fact that $\left|2\E[Y| X ] - 1\right| \ge h$ almost surely implies
\begin{equation}
\label{riskcond}
\Pr(f^*_B(X) \neq Y|X) \le \frac{1}{2}(1 - h).
\end{equation}
At first, Theorem \ref{mainthmproc} implies that, with probability $1 - \delta$, we have
\begin{equation}
\label{riskb}
R_{h/2}(\widehat{f}_{h/2}) - R(f^*) \le c_1\frac{d\log\frac{n}{d} + \log \frac{1}{\delta}}{nh},
\end{equation}
where $c_1 > 0$ is an absolute constant. Since the diameter of the class is finite, there are at most $D$ instances in the domain $\mathcal{X}$ such that $\widehat{f}_{h/2} = *$. This follows from the way we construct the classifier in Theorem \ref{mainthmproc}.  We want to use the third part of the sample to specify the prediction on these instances. From now on, we work conditionally on the first and the second parts of the sample and therefore the $D$ instances of interest are assumed to be fixed. This set of instances will be denoted by $\{x_1, \ldots, x_D\}$ without loss of generality.

Consider the class $\mathcal{C}$ which consists of functions that shatter the set $\{x_1, \ldots, x_D\}$ and are equal to zero everywhere else. There are exactly $2^D$ functions and they form a class with VC dimension $D$. Consider the following joint distribution of $X$ and $Y$ (denoted by $\tilde{P}$ in what follows):
\begin{itemize}
\item The distribution of $X$, that is $P_X$, remains unchanged.
\item For $x \in \{x_1, \ldots, x_D\}$ the distribution $Y|x$ is also unchanged.
\item For $x \notin \{x_1, \ldots, x_D\}$ we set $Y|x = 0$.
\end{itemize}
It is easy to see that $S_3$ can be used to learn the class $\mathcal{C}$ with respect to the distribution $\tilde{P}$. Observe that for $\tilde{P}$ the Bayes optimal rule $f^*_{B, \tilde{P}}$ is in $\mathcal{C}$. Moreover, in this case $\text{maj}_{S_3}$ is ERM over $\mathcal{C}$. Finally, by \citep[Inequality $19$]{Boucheron05b} we have, with probability at least $1 - \delta$, 
\begin{equation}
\label{majority}
R(\text{maj}_{S_3}) - R(f^*_{B, \tilde{P}}) \le c_2\frac{D + \log \frac{1}{\delta}}{nh},
\end{equation}
where the risk is defined with respect to $\tilde{P}$ and $c_2 > 0$ is a numerical constant. Observe that
\[
R(\text{maj}_{S_3}) = \Pr(\text{maj}_{S_3}(X) \neq Y\ \text{and}\ X \in  \{x_1, \ldots, x_D\}),
\]
and by \eqref{riskcond}, conditioned on the first two parts of the sample, we have
\[
R(f^*_{B, \tilde{P}}) \le \frac{1}{2}(1 - h)\Pr(\{x_1, \ldots, x_D\}).
\]
Using the union bound and combining \eqref{riskb} and \eqref{majority} we have
\begin{align*}
R(\widehat{f}) &= \Pr(\widehat{f}(X) \neq Y\ \text{and}\ X \notin  \{x_1, \ldots, x_D\})+\Pr(\text{maj}_{S_3}(X) \neq Y\ \text{and}\ X \in  \{x_1, \ldots, x_D\})
\\
&\le \Pr(\widehat{f}_{h/2}(X) \neq Y\ \text{and}\ \widehat{f}_{h/2}(X)\in \{0, 1\}) + \frac{1}{2}(1 - h)\Pr(\{x_1, \ldots, x_D\}) + c_2\frac{D + \log \frac{1}{\delta}}{nh}
\\
&= R_{h/2}(\widehat{f}_{h/2}) + c_2\frac{D + \log \frac{1}{\delta}}{nh} \le R(f^*) + c_2\frac{D + \log \frac{1}{\delta}}{nh} + c_1\frac{d\log\frac{n}{d} + \log \frac{1}{\delta}}{nh}.
\end{align*}
The claim follows. \qed

\subsection{Proof of Corollary \ref{finitediametercorollary}}

We highlight only the steps that are different from the proof of Theorem \ref{thm:finitediameter}. Using the notation above we have
\[
R(\widehat{f}) = \Pr(\widehat{f}_{1/2}(X) \neq Y\ \text{and}\ \widehat{f}_{1/2}(X) \neq *)+\Pr(\widetilde{g}(X) \neq Y\ \text{and}\ \widehat{f}_{1/2}(X) = *).
\]
By Theorem \ref{mainthmproc} we have, with probability at least $1 - \delta$,
\[
\Pr(\widehat{f}_{1/2}(x) \neq Y\ \text{and}\ \widehat{f}_{1/2}(X) \neq *) - R(f^*) \lesssim \frac{d\log \frac{n}{d} + \log \frac{1}{\delta}}{n} 
\]
It is only left to prove that, with probability at least $1 - \delta$,
\[
\Pr(\widetilde{g}(X) \neq Y\ \text{and}\ \widehat{f}_{1/2}(X) = *) \lesssim \frac{D_{P_X}(n) + \log \frac{1}{\delta}}{n}.
\]
Recall the notation of the proof of Theorem \ref{thm:finitediameter}. Observe that for the distribution $\tilde{P}$ the Bayes optimal rule $f^*_{B, \tilde{P}}$ is in $\mathcal{C}_{\widehat{g}, \widehat{g}_1}$. Moreover, since $h = 1$ for this distribution we have $Y = f^*_{B, \tilde{P}}$ almost surely. Therefore, our problem corresponds to the realizable (the Bayes optimal rule is in the class and there is no noise) case classification. To analyze this problem, we fix $\varepsilon > 0$ and consider an $\varepsilon$-net of $\mathcal{C}_{\widehat{g}, \widehat{g}_1}$ with respect to the $L_1(\tilde{P}_X)$ metric. Observe that $\widetilde{g}$ \eqref{ermovernet} corresponds to ERM over this set for $\varepsilon = r$ defined by \eqref{radius}. The corresponding risk bound is well known and is immediately implied by \cite[Theorem 5]{bshouty2009using}. This result claims, in particular, that if 
\begin{equation}
\label{ncondition}
n \ge \frac{64\log \mathcal{N}(\mathcal{C}_{\widehat{g},\widehat{g}_1 }, \varepsilon, L_{1})}{\varepsilon} + \frac{32\log \frac{1}{\delta}}{\varepsilon},
\end{equation}
then, with probability at least $1 - \delta$, we have
\[
R_{\tilde{P}}(\widetilde{g}) = \Pr(\widetilde{g}(X) \neq Y\ \text{and}\ \widehat{f}_{1/2}(X) = *) \le \varepsilon.
\]
The claim follows once we set $\varepsilon = r$ and tune the constants in Algorithm \ref{alg:finitediameterdistibution} in a way such that \eqref{ncondition} is satisfied.
\qed

\subsection{Proof of Proposition \ref{prop:logfactor}}
\label{sec:logfactorproof}
At first we recall that $h = 1$ corresponds to the case $Y = f^*_B(X)$ almost surely. Assume without loss of generality that $f = 0$ belongs to $\F$. Given a sample $S = ((X_1, Y_1), \ldots, (X_{n + 1}, Y_{n + 1}))$, denote by $S^{(i)}$ the sample $S$ with hidden $i$-th label, that is 
\[
S^{(i)} = ((X_1, Y_1), \ldots, (X_{i- 1}, Y_{i - 1}), (X_i, ?), (X_{i+ 1}, Y_{i + 1})\ldots, (X_{n + 1}, Y_{n + 1})).
\]
Let $\widehat{f}_{S^{(i)}}$ denote the classifier $\widehat{f}$, which was trained on the sample $S^{(i)}$. The leave-one-out error is defined by
\[
\text{LOO} = \frac{1}{n + 1}\sum\limits_{i = 1}^{n + 1}\ind[\widehat{f}_{S^{(i)}}(X_i) \neq Y_{i}].
\]
It is well-known \cite{Haussler94} that $\E\text{LOO} = \E R(\widehat{f})$. We consider the following learning algorithm. Given the sample $S$ it outputs $\widehat{f}$ defined as follows
\[
\widehat{f}(x) = \begin{cases} Y_i, & \mbox{if}\ x = X_i\  \mbox{for some}\ X_i \in S, \\ 0, &\mbox{otherwise}. \end{cases}
\]
Since there is no noise in the labeling, our algorithm is correctly defined. Let $S^u$ be a subset of $S$ such that any $(X_i, Y_i) \in S$ is presented only once. This set is obtained from $S$ by removing all excess copies. We have
\begin{align*}
(n + 1)\text{LOO} &= \sum\limits_{i = 1}^{n + 1}\ind[\widehat{f}_{S^{(i)}}(X_i) \neq Y_{i}] 
=  \sum\limits_{(X_i, Y_i) \in S^u}\ind[\widehat{f}_{S^{(i)}}(X_i) \neq Y_{i}]
=  \sum\limits_{(X_i, Y_i) \in S^u}\ind[Y_{i}=1]
\\
&\le D + \inf\limits_{f \in \F}\sum\limits_{(X_i, Y_i) \in S^u}\ind[f(X_i) \neq Y_i] \le D + \inf\limits_{f \in \F}\sum\limits_{i = 1}^{n + 1}\ind[f(X_i) \neq Y_i]
\\
&\le D + \sum\limits_{i = 1}^{n + 1}\ind[f^*(X_i) \neq Y_i]\,.
\end{align*}
By taking the expectation with respect to both sides of the last inequality and using $\E\ \text{LOO} = \E R(\widehat{f})$ we prove the claim.
\qed

\section{Related Work}
\label{sec:relwork}
\paragraph{Selective classification and abstentions.}
One of the related directions is the selective classification setup studied extensively in \cite{Elyaniv10, Wiener15, Elyaniv17}. The idea of selective classification is roughly the following: if we decide to classify an instance $x$, then we should always output $f^{*}(x)$. In particular, in the realizable case, one should always output the true label. Simultaneously, we want to minimize the total mass of the region of abstentions. Our aim is slightly less ambitious, and therefore our results may not be directly compared with the results in selective classification. However, our analysis does not require any assumptions presented in the selective classification literature such as realizability or the low noise assumptions \cite{Elyaniv17}.  

The statistical analysis of the Chow's model was presented in \cite{Wegkamp06, Bartlett08} among other papers. Notice that there are some important differences: for example, in \cite{Wegkamp06} the authors assume that we are given the class $\F^{\prime}$ of $\{0, 1, *\}$-valued functions and are interested in the analysis of ERM over $\F^{\prime}$ with respect to Chow's risk \cite{Chow70}. To get the fast learning rates with respect to this risk a special version of Tsyabkov's assumption was introduced for classes of $\{0, 1, *\}$-valued functions. Our results are of a slightly different flavour: we start with a class $\F$ of $\{0, 1\}$-valued functions and build a class $\widehat{\G}$ of $\{0, 1, *\}$-valued functions in order to avoid any margin assumptions.

Finally, we discuss the application of aggregation theory in the context of classification with a reject option studied by Freund, Mansour and Shapire in \cite{Freund04}. The authors consider the weighted majority type algorithm. A direct comparison with our results seems difficult, because the analysis in \cite{Freund04} is tuned to finite classes and, more importantly, the risk bounds relate the risk of the algorithm to twice the risk of the best function in the class. It is by now well-known that in this case ERM always has $\tilde{O}\left(\frac{d}{n}\right)$ rates of convergence (this result follows from \cite[Corollary $5.3$]{Boucheron05b}). Finally, their results are not adaptive to various low noise assumptions. We additionally  refer to \cite[Section $10$]{Elyaniv10} for the detailed discussion.

\paragraph{Intermediate rates.}
Most of the results obtained in the bounded noise regime can be extended to the case where the noise is not uniformly bounded but the way it approaches $1/2$ is controlled. This is the so-called Tsybakov's condition \cite{Tsybakov04} which also assumes well-specification of the model, that is, $f^*_B \in \F$ and leads to rates that are in between $O\left(1/\sqrt{n}\right)$ and $O\left(1/n\right)$. To simplify our exposition we focus on the bounded noise regime but note that our results can potentially be extended to this more general case. We see this this direction as an interesting extension of our work.

Note that there is also a general \emph{Bernstein} condition \cite{Bartlett06} which can give the fast rates and which is implied by Massart's or Tsybakov's conditions, but there are no natural situations in the misspecified classification setting where this condition is known to be satisfied. See some related discussions in \cite[Section 1]{Elyaniv17} and our results in Section \ref{sec:ermrates}.

\paragraph{Computational Aspects.}
In this paper we focus mainly on the statistical analysis of the binary classification and provide the algorithms that perform at least as good as ERM with respect to the binary loss. In particular, we need to compute ERM with respect to the binary loss as an intermediate step of our algorithms. However, it is well-known that computational problems for computing ERM appear even in the realizable classification. We refer to several hardness results presented in the book \citep{anthony2009neural}.  We remark that some of our results can potentially be made more computationally friendly in some cases via the surrogate losses (see e.g, \citep{zhang2004statistical, bartlett2006convexity, wegkamp2011support}). This is an interesting direction of future research. 
\\
\\
\textbf{Acknowledgments.} We would like to thank Ilya Tolstikhin and the anonymous reviewers for many suggestions that significantly improved the presentation of this paper.

{\footnotesize
\bibliography{mybib}
}

\appendix
\section{On Lemma \ref{pairwise}}
\label{app:ratiotype}
In this section we discuss different proofs of Lemma \ref{pairwise} and some corresponding results. One of the implications of our analysis, which to the best of our knowledge was not observed in the previous literature, is that there is a short way of proving almost optimal risk bounds under various margin conditions using only on the original ratio-type estimates of Vapnik and Chervonenkis \cite{Vapnik74} which apply to binary-valued functions. This is quite remarkable, because the existing proofs in the binary case \cite{Bartlett06, Koltchinskii06, Massart06} exploit Talagrand's inequality and some technical results from empirical process theory.

At first, the class $\mathcal{H} = \{x \to \ind[f(x) \neq g(x)]: f, g \in \mathcal {F}\}$ has the VC dimension at most $10d$ \cite{Vidyasagar03} (we may alternatively control the growth function). Now by the classical ratio-type estimates of \cite{Vapnik74, Boucheron05b} we have, with probability at least $1 - \delta$,
\begin{equation}
\label{ratiotype}
\frac{P|f - g| - P_n|f - g|}{\sqrt{P|f - g|}} \lesssim \sqrt{\frac{d\log \frac{n}{d} + \log \frac{1}{\delta}}{n}},
\end{equation}
and
\[
\frac{P_n|f - g| - P|f - g|}{\sqrt{P_n|f - g|}} \lesssim \sqrt{\frac{d\log \frac{n}{d} + \log \frac{1}{\delta}}{n}}.
\]
To prove \eqref{pairwisedist} and \eqref{pairwisedistsec} we use that for $a, b, c \ge 0$ the inequality $a \le b\sqrt{a} + c$ implies $a \le b^2 + b\sqrt{c} + c$ (see \cite{Boucheron05b}). Both inequalities now follow from simple algebra. 

To the best of our knowledge the simplest proof of \eqref{starshaped} follows from the idea in \cite{Dasgupta07, Hsu10} and is essentially presented there. We reproduce this short argument for the sake of completeness. Consider the class of $\{0, 1\}$-valued functions 
\[
\F_{1} = \{(x, y) \to \ind[f(x) \neq y\ \text{and}\ f^*(x) = y]: f \in \F\}.
\]
and
\[
\F_2 = \{(x, y) \to \ind[f(x) = y\ \text{and}\ f^*(x) \neq y]: f \in \F\}.
\]
We have for $h_{f} \in \F_1$ and $g_{f} \in \F_2$ that $h_f(x, y) + g_f(x, y) = \ind[f(x) \neq f^*(x)]$ and the VC dimension of both classes is at most $d$. Moreover, $h_f(x, y) - g_f(x, y) = \ind[f(x) \neq y] - \ind[f^*(x) \neq y]$. Applying the ratio-type estimates of \cite{Vapnik74, Boucheron05b} to $\F_1$ and $\F_2$ (as we did in the proof of \eqref{pairwisedist} and \eqref{pairwisedistsec})  and combining the results with the union bound we have for any $h \in \mathcal{L}_q$ and the corresponding $f \in \F$, with probability at least $1 - \delta$,
\begin{align*}
|Ph - P_n h| &= |P(h_f - g_f) - P_n(h_f - g_f)| \le |Ph_f - P_n h_f| + |Pg_f - P_n g_f|
\\
&\lesssim \alpha(\sqrt{P_n h_f} + \sqrt{P_n g_f}) + \alpha^2 \lesssim \alpha\sqrt{P_n |f - f^*|} + \alpha^2,
\end{align*}
where we used $\sqrt{P_n h_f} + \sqrt{P_n g_f} \le \sqrt{2P_n(h_f + g_f)}$.
The inequality \eqref{starshaped} follows.

In particular, combining \eqref{starshaped} with \eqref{pairwisedistsec} we have for any ERM (denoted by $\widehat{g}$) over $\F$, with probability at least $1 - \delta$,
\[
R(\widehat{g}) - R(f^*) \lesssim \alpha^2 + \alpha\sqrt{P|\hat g - f^*|}. 
\]
Under Assumption \ref{bernstein} this implies immediately that, with probability at least $1 - \delta$,
\[
R(\widehat{g}) - R(f^*) \lesssim \left(\frac{Bd\log \frac{n}{d} + B\log \frac{1}{\delta}}{n}\right)^{\frac{1}{2 - \beta}}.
\]

An alternative and self-contained proof of Lemma \ref{pairwise} follows from the general techniques in empirical process theory. We give a simplified version of the argument showing a version of the inequality \eqref{starshaped} and note that the remaining inequalities can be similarly obtained. That is, we show that simultaneously for all $h \in \mathcal{L}_q$ and the corresponding $f \in \F$, with probability at least $1 - \delta$, it holds that 
\[
|P h - P_n h| \lesssim \sqrt{P|f - f^*|\frac{d\log \frac{n}{d} + \log \frac{1}{\delta}}{n}} + \frac{d\log \frac{n}{d} + \log \frac{1}{\delta}}{n}.
\]
In what follows, we do not use the ratio-type estimates of \cite{Vapnik74, Boucheron05b}. Recall the definition \eqref{lq} and observe that $0 \in \mathcal{L}_1$. 
Consider the star-shaped hull of $\mathcal{L}_1$ around zero, that is, the class $\mathcal{H}$ of $[0,1]$-valued functions defined as
\[
\mathcal{H} = \{\beta f: f \in \mathcal{L}_1, \beta \in [0, 1]\}~.
\]
 Consider the following fixed point
\[
\gamma(\eta, \delta) = \inf\left\{s \ge 0: \Pr\left(\sup\limits_{h \in \mathcal{H}, Ph^2 \le s^2}|(P - P_n)h| \le \eta s^2\right) \ge 1 - \delta\right\},
\]
where $\eta$ is a numerical constant to be specified later. By the definition we have, with probability at least $1 - \delta$,
\[
\sup\limits_{h \in \mathcal H, \sqrt{Ph^2} \le \gamma(\eta, \delta)}|(P - P_n)h| \le \eta \gamma(\eta, \delta)^2.
\]
Now consider any $h \in \mathcal H$ such that $\sqrt{P h ^2} \ge \gamma(\eta, \delta)$. By the star-shapedness we have for any $h^{\prime} \in \mathcal H$ defined by $\frac{\gamma(\eta, \delta)h}{\sqrt{P h ^2}}$ that 
\[
|(P - P_n)h^{\prime}| \le \eta \gamma(\eta, \delta)^2,
\]
which implies for any $h \in \mathcal H$,
\begin{equation}
\label{fixedpoin}
|(P - P_n)h| \le \eta \gamma(\eta, \delta)\sqrt{Ph^2} + \eta\gamma(\eta, \delta)^2.
\end{equation}
Observe that the following holds for any $h \in \mathcal{H}$ and the corresponding $f \in \F$,
\[
Ph^2 = \beta^2P(|f(X) - Y| - |f^*(X) - Y|)^2 \le P|f - f^*|^2 \le P|f - f^*|,
\]
where $\beta \in [0, 1]$ comes from the star-shapedness. Finally, we provide an upper bound on $\gamma(\eta, \delta)$ and choose the value of $\eta$. In what follows, we prove that, with probability at least $1 - \delta$,
\begin{equation}
\label{eq:desiredpoint}
\sup\limits_{h \in \mathcal{H}, Ph^2 \le s^2}|(P - P_n)h| \lesssim s\sqrt{\frac{d\log\frac{n}{d} + \log\frac{1}{\delta}}{n}} + \frac{d\log\frac{n}{d} + \log\frac{1}{\delta}}{n}.
\end{equation}
For $\mathcal H$ let $\mathcal{N}(\mathcal H, \varepsilon,  L_2(P_n))$ denote the covering number with respect to the distance induced by $L_2(P_n)$ norm.  By Haussler's Lemma \cite{Haussler95} (we refer to \cite{Zhivotovskiy17} for a simplified proof of this result) we have for any VC class $\F$, $\mathcal{N}(\F, \varepsilon, L_2(P_n)) \le e(d + 1)\left(\frac{2e}{\varepsilon^2}\right)^d$. By the standard argument \cite{Bartlett06} for star-shaped hulls we have
\begin{equation}
\label{coveringnum}
\mathcal{N}(\mathcal H, \varepsilon,  L_2(P_n)) \le e(d + 1)\left(\frac{8e}{\varepsilon^2}\right)^d\left(1 + \left\lceil\frac{2}{\varepsilon}\right\rceil\right).
\end{equation}
By the symmetrization argument \cite{Talagrand14} we have
\begin{equation}
\label{eq:symmetrization}
\sup\limits_{h \in \mathcal{H}, Ph^2 \le s^2}|(P - P_n)h| \le 2\E\sup\limits_{h \in \mathcal{H}, Ph^2 \le s^2}\left|\frac{1}{n}\sum\limits_{i = 1}^n \varepsilon_i h(X_i)\right|,
\end{equation}
where $(\varepsilon_i)_{i = 1}^n$ are independent Rademacher random variables. Recall the notation of the $L_s$ diameter \eqref{eq:diam} and denote
\[
\mathcal{D}_n = \mathcal{D}(\{h \in \mathcal{H}: Ph^2 \le s^2\}, L_2(P_n)).
\]
Using \eqref{coveringnum} and Dudley's integral arguments for Bernoulli processes (see e.g., \citep[Lines of the proof of Lemma 13.5]{Boucheron13}) we have
\begin{align}
\E\sup\limits_{h \in \mathcal{H}, Ph^2 \le s^2}\left|\sum\limits_{i = 1}^n \varepsilon_i h(X_i)\right|
&\lesssim \sqrt{n}\E\int\limits_{0}^{\mathcal{D}_n}\sqrt{d\log \frac{e}{r}}dr \nonumber
\\
&\lesssim \sqrt{n}\E\mathcal{D}_n\sqrt{d\log \frac{e}{\mathcal{D}_n}}\left(\ind[\mathcal{D}_n \ge \sqrt{d/n}] + \ind[\mathcal{D}_n < \sqrt{d/n}]\right) \nonumber
\\
&\lesssim \sqrt{n}\E\mathcal{D}_n\sqrt{d\log \frac{n}{d}} + d\sqrt{\log \frac{n}{d}}~.
\label{eq:symmetrproc}
\end{align}
By Jensen's inequality combined with the symmetrization and contraction inequalities we have
\begin{align}
\sqrt{n}\E\mathcal{D}_n &\lesssim \sqrt{\E\sup\limits_{h \in \mathcal{H}, Ph^2 \le s^2}\sum_{i = 1}^{n} h^2(X_i)}\nonumber
\\
&\lesssim \sqrt{\E\sup\limits_{h \in \mathcal{H}, Ph^2 \le s^2}\sum_{i = 1}^{n} (h^2(X_i) -\E h^2(X_i)) + s^2} \nonumber
\\
&\lesssim \left(\sqrt{\E\sup\limits_{h \in \mathcal{H}, Ph^2 \le s^2}\left(\sum_{i = 1}^n\varepsilon_i h(X_i)\right)}+ s\right)~.
\label{eq:symmetrprocsec}
\end{align}
Combining \eqref{eq:symmetrization}, \eqref{eq:symmetrproc} and \eqref{eq:symmetrprocsec} we obtain
\[
\E\sup\limits_{h \in \mathcal{H}, Ph^2 \le s^2}|(P - P_n)h| \lesssim s\sqrt{\frac{d\log\frac{n}{d}}{n}} + \frac{d\log\frac{n}{d}}{n}.
\]
Finally, we apply Talagrand's concentration inequality \cite[Theorem 12.2]{Boucheron13} to the process $|(P - P_n)h|$ indexed by $\{h \in \mathcal{H}: Ph^2 \le s^2\}$ and prove \eqref{eq:desiredpoint}.
Fixing an appropriate numerical constant $\eta > 0$, we have
\[
\gamma(\eta, \delta) \lesssim \sqrt{\frac{d\log \frac{n}{d} + \log \frac{1}{\delta}}{n}}.
\]
The claim follows immediately by \eqref{fixedpoin}.

\end{document}